\documentclass[conference,compsoc]{IEEEtran}

\usepackage{cite}
\usepackage{amsthm}
\usepackage{textcomp}
\usepackage{xcolor}
\usepackage{wrapfig}

\usepackage{multirow}
\usepackage{makecell}
\usepackage{dsfont}
\usepackage{amssymb}
\usepackage{amsthm}
\usepackage{amsmath}
\usepackage{bbm}

\usepackage{enumitem}
\usepackage{nicefrac}
\usepackage{booktabs}
\usepackage{hyperref}
\usepackage{caption}
\usepackage{subcaption}
\usepackage{graphicx}
\usepackage[export]{adjustbox}

\usepackage{algorithm}
\usepackage{algorithmic}

\usepackage{cite}
\newcommand{\cB}{\mathcal{B}}
\newcommand{\cC}{\mathcal{C}}
\newcommand{\cD}{\mathcal{D}}

\newcommand{\cN}{\mathcal{N}}
\newcommand{\cO}{\mathcal{O}}

\newcommand{\cU}{\mathcal{U}}

\newcommand{\cX}{\mathcal{X}}

\newcommand{\cZ}{\mathcal{Z}}

\newcommand\bR{\mathbb{R}}
\newcommand\bN{\mathbb{N}}
\newcommand\bE{\mathbb{E}}
\newcommand\Prob{\mathbb{P}}

\newcommand{\Id}{\mathds{1}}

\newcommand{\abs}[1]{\lvert#1\rvert}

\newtheorem{thm}{Theorem}
\newtheorem{cor}{Corollary}
\newtheorem{lem}{Lemma}

\theoremstyle{plain}

\def\BibTeX{{\rm B\kern-.05em{\sc i\kern-.025em b}\kern-.08em
    T\kern-.1667em\lower.7ex\hbox{E}\kern-.125emX}}

\begin{document}

\title{RAB: Provable Robustness Against Backdoor Attacks\vspace{-1.5em}}
\author{
Maurice Weber$^{\dagger}$
$^*$~~~~~Xiaojun Xu$^{\ddagger}$ $^*$~~~~~Bojan Karla\v{s}$^\dagger$~~~~~Ce Zhang$^\dagger$~~~~~Bo Li$^\ddagger$\\
\small $^\dagger$ ETH Zurich, Switzerland~~~~\{maurice.weber, karlasb, ce.zhang\}@inf.ethz.ch\\
\small $^\ddagger$ University of Illinois at Urbana-Champaign, USA~~~~\{xiaojun3, lbo\}@illinois.edu
}

\maketitle

\begin{abstract}
Recent studies have shown that deep neural networks (DNNs) are vulnerable to adversarial attacks, including evasion and backdoor (poisoning) attacks.
On the defense side, there have been intensive efforts on improving both empirical and provable robustness against evasion attacks; however, the provable robustness against backdoor attacks still remains largely unexplored.
In this paper, we focus on certifying the machine learning model robustness against general threat models, especially backdoor attacks.
We first provide a unified framework via randomized smoothing techniques and show how it can be instantiated to certify the robustness against both evasion and backdoor attacks. We then propose the \textit{first} robust training process,  RAB, to smooth the trained model and certify its robustness against backdoor attacks. We prove the robustness bound for machine learning models trained with RAB and prove that our robustness bound is tight.
In addition, we theoretically show that it is possible to train the robust smoothed models efficiently for simple models such as K-nearest neighbor classifiers, and we propose an exact smooth-training algorithm that eliminates the need to sample from a noise distribution for such models.
Empirically, we conduct comprehensive experiments for different machine learning (ML) models such as DNNs, support vector machines,
and K-NN models on MNIST, CIFAR-10, and ImageNette datasets and provide the first benchmark for certified robustness against backdoor attacks.
In addition, we evaluate K-NN models on a spambase tabular dataset to demonstrate the advantages of the proposed exact algorithm.
Both the theoretic analysis and the comprehensive evaluation on diverse ML models and datasets shed light on further robust learning strategies against general training time attacks.
\end{abstract}

\section{Introduction}

Building machine learning algorithms that
are robust to adversarial attacks has been
an emerging topic over the last decade.
There are mainly two different types of
adversarial attacks: (1) \textit{evasion attacks},
in which the attackers manipulate the
test examples against a trained machine learning (ML) model, and (2) \textit{data poisoning attacks},
in which the attackers are allowed to perturb the
training set. Both types of
attacks have attracted
intensive interests from
academia as well as industry~\cite{goodfellow2014explaining,xiao2018generating,liao2018backdoor,yang2017generative}.

In response, several empirical solutions have been proposed as defenses against evasion attacks~\cite{carlini2017adversarial,xu2017feature,ma2018characterizing,yang2018characterizing}.
For instance, adversarial training has been proposed to retrain the ML models with generated adversarial examples~\cite{madry2017towards}; quantization has been applied to either inputs or neural network weights to defend against potential adversarial instances~\cite{xu2017feature}.
However, recent studies have shown that these defenses are not resilient against intelligent adversaries responding dynamically to the deployed defenses~\cite{carlini2017adversarial,athalye2018obfuscated}.

As a result, one recent, exciting line of research aims to develop \textit{certifiably robust} algorithms against {\em evasion attacks}, including both deterministic and probabilistic certification approaches~\cite{li2020sok}.
In particular, among these certified robustness approaches, only randomized smoothing and its variations are able to provide certified robustness against evasion attacks on large-scale datasets such as ImageNet~\cite{lecuyer2019,cohen2019certified,yang2020}.
Intuitively, the randomized smoothing-based approaches are able to certify the robustness of a smoothed classifier, by outputting a consistent prediction for an adversarial input as long as the perturbation is within a certain radius.
The smoothed classifier is obtained by taking the expectation over the possible outputs given a set of randomized inputs which are generated by adding noise drawn from a certain distribution.

\begin{figure}[t!]
\centering
\includegraphics[width=0.49\textwidth]{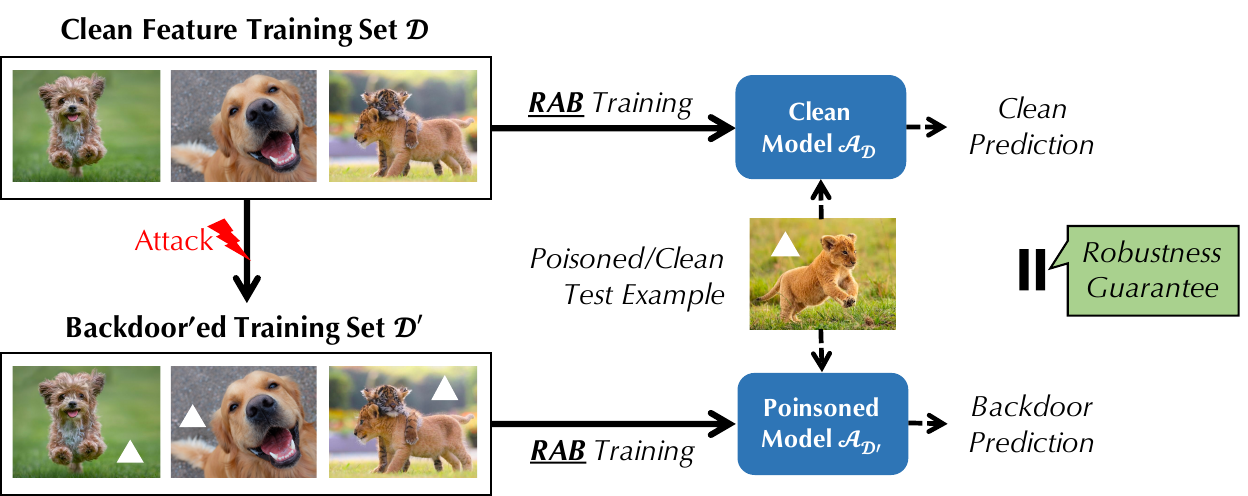}
\caption{In this paper, we define a robust training process RAB against backdoor attacks.
Given a poisoned dataset $\cD'$ ---
produced by adding backdoor patterns
$\Delta$
to some instances in
the dataset $\cD$ with  clean
features  ---
this robust training process guarantees that,
for all
test examples $x$, $\mathcal{A}_{\mathcal{D}'}(x) = \mathcal{A}_{\mathcal{D}}(x)$, with
high probability when the magnitude of the backdoor pattern $\Delta$ is
within the certification radius.}
\label{fig:framework}
\end{figure}

Despite these recent developments on certified robustness against \textit{evasion attacks}, only empirical studies have been conducted to defend against \textit{backdoor attacks}~\cite{wang2019neural,gao2019strip,gu2017badnets,li2016data}, and the question of how to improve and certify the robustness of given machine learning models against backdoor attacks remains largely unanswered.
To the best of our knowledge, there is no certifiably robust strategy
to deal with backdoor attacks yet. Naturally,
we ask: {\em Can we develop certifiably
robust ML models against backdoor attacks?}

It is clear that extending existing certification methods against evasion attacks to certifying training-time attacks is challenging given these two significantly different threat models.
For instance, even certifying a label flipping training-time attack is non-trivial as illustrated in a concurrent work~\cite{rosenfeld2020certified}, which proposes to certify against a label flipping attack by setting a limit to how many labels in the training set may be flipped such that it does not affect the final prediction leveraging randomized smoothing.
As backdoor attacks involve both label flipping and instance pattern manipulations, providing certifications can be even more challenging.

In particular, to carry out a backdoor attack,
an attacker adds small backdoor patterns to a subset of training instances such that
the trained model is biased toward test images with
the same patterns~\cite{gu2019badnets,chen2017targeted}.
Such attacks can be applied to various real-world scenarios such as online face recognition systems~\cite{li2016data,chen2017targeted}.
In this paper, we present the first certification process, referred to as RAB, which offers provable robustness for ML models against backdoor attacks.
As shown in Figure~\ref{fig:framework}, our \textbf{certification goal} is to guarantee \textit{that a test instance, which may contain backdoor patterns, will be classified the same, independent of whether the models were trained on data with or without backdoors,
as long as the embedded backdoor patterns are within an $L_p$-ball of radius $R$.}
We formally define
the corresponding threat model and
our certification goal
in Section~\ref{s:threat}.

Our approach to achieving this is mainly inspired
by randomized
smoothing, a technique to certify robustness against evasion
attacks~\cite{cohen2019certified}, but goes
significantly beyond it due to the different settings (e.g. evasion and backdoor attacks). Our \textbf{first
step/contribution} is to develop
a general theoretical framework to generalize randomized smoothing
to a much larger family of
functions and smoothing distributions.
This allows us to support cases
in which a classifier is a function that
takes as input a test instance \emph{and} a training set.
With our framework, we can \emph{(1) provide robustness certificates against both evasion and dataset poisoning attacks;} \emph{(2) certify any classifier which takes as input a tuple of test instance and training dataset} and \emph{(3) prove that the derived robustness bound is tight}.
Given this general framework,
we can enable a
basic version of the proposed RAB framework.
At a high level, as shown in Figure~\ref{fig:process}, given training set $\cD$, RAB generates $N$ additional ``smoothed" training sets
$\cD+\epsilon_i$ by adding noise $\epsilon_i$ ($i \in \{1,\ldots,N\}$) drawn from a certain smoothing distribution and, for each of these $N$ training sets, a corresponding classifier is trained resulting in an ensemble of $N$ different classifiers.
These models are then aggregated to generate a ``smoothed classifier" for which we prove that its output will be consistent regardless of whether
there are backdoors added during training, as long as the backdoor patterns satisfy certain conditions.

However, this basic version
is not enough
to provide
satisfactory certified robustness against backdoor attacks.
When we instantiate
our theoretical framework with a practical training pipeline to provide certified robustness against backdoor attacks,
we need to further develop nontrivial techniques
to improve two aspects: (1) Certification Radius and
(2) Certification Efficiency.
Our \textbf{second
step/contribution} are two
non-trivial technical
optimizations.
(1) To improve the \textit{certification radius},
we certify DNN classifiers with
a data augmentation
the scheme enabled by hash functions and, in the meantime,
explore
different design decisions
such as the smoothness of the training process.
This provides additional guidance for improving the certified robustness against backdoor attacks and we hope that it can inspire
other researches in the future.
(2) To improve the \textit{certification
efficiency}, we
observed that for certain families
of classifiers, namely $K$-nearest
neighbor classifiers,
we can develop an efficient algorithm
to compute the smoothing
result \textit{exactly, eliminating the need to resort to Monte Carlo algorithms
as for generic classifiers}.

Our \textbf{third contribution} is an extensive
benchmark, evaluating our framework RAB
on multiple
machine learning models and provide
the first collection of certified robustness bounds
on a diverse range of datasets, namely MNIST, CIFAR-10, ImageNette, as well as spambase tabular data.
We hope that these experiments and benchmarks can provide future directions for improving the robustness of ML models against backdoors.

Being the first result
on certified robustness against backdoor attacks, we believe that these results can be further improved by future research endeavours inspired by this work. We make the code and evaluation protocol
publicly available with the hope to facilitate future
research by the community.

\vspace{0.5em}
\noindent{\bf Summary of Technical Contributions.}
Our technical contributions are as follows:
\begin{itemize}[leftmargin=*]
    \item We propose a unified framework to certify the model robustness against both evasion and backdoor attacks and prove that our robustness bound is tight.
    \item We provide the first certifiable robustness bound for general machine learning models against backdoor attacks considering \textit{different} smoothing noise distributions.
    \item We propose an \textit{exact} {efficient} smoothing algorithm for $K$-NN models without needing to sample random noise during training.
    \item We conduct extensive reproducible large-scale experiments and provide a benchmark for certified robustness against three representative backdoor attacks for multiple types of models (e.g., DNNs, support vector machines, and $K$-NN) on diverse datasets. We also provide a series of ablation studies to further analyze the factors that affect model robustness against backdoors.
\end{itemize}

\vspace{0.5em}
\noindent{\bf Outline.}
The remainder of this paper is organized as follows. Section~\ref{s:back} provides background on backdoor attacks and related verifiable robustness techniques, followed by the threat model and method overview in Section~\ref{s:threat}. Section~\ref{sec:framework} presents the proposed general theoretical framework for certifying robustness against evasion and poisoning attacks, the tightness of the derived robustness bound, and sheds light on a connection between statistical hypothesis testing and certifiable robustness. Section~\ref{sec:backdoor_attacks} explains in detail the proposed approach RAB for certifying robustness against backdoor attacks under the general framework with Gaussian and uniform noise distributions. Section~\ref{sec:models} analyzes the robustness properties of DNNs and $K$-NN classifiers and presents algorithms to certify robustness for such models (mainly with binary classifiers). Experimental results are presented in section~\ref{sec:experiments}. Finally, Section~\ref{sec:related_work} puts our results in context with existing work, Section~\ref{sec:limitations} discusses the limitations of our work,
and Section~\ref{sec:conclusions} concludes.

\section{Background}
\label{s:back}

In this section, we provide an overview of different backdoor attacks and briefly review the randomized smoothing technique for certifying robustness against evasion attacks.

\subsection{Backdoor attacks}

A backdoor attack
aims to inject certain ``backdoor" patterns into the training set and associate such patterns with a specific adversarial target (label). As a result, during testing time, any test instance with such a pattern will be misclassified as the preselected adversarial target~\cite{gu2017badnets,chen2017targeted}.
ML models with injected backdoors are called \emph{backdoored models} and they are typically able to achieve performance similar to clean models on benign data, making it challenging to detect whether the model has been backdoored.

There are several ways to categorize backdoor attacks. First, based on the \emph{adversarial target design}, the attacks can be characterized either as \emph{single target attacks} or \emph{all-to-all attacks}.
In a \emph{single target attack}, the backdoor pattern will cause the poisoned classifier to always return a designed target label. An \emph{all-to-all} attack leverages the backdoor pattern to permute the classifier results.

The second categorization is based on \emph{different types of backdoor patterns}. There are \emph{region based} and \emph{blending} backdoor attacks.
In the \emph{region based} attack, a specific region of the training instance is manipulated in a subtle way that will not cause human notification~\cite{gu2017badnets,liao2018backdoor}.
In particular, it has been shown that such backdoor patterns can be as small as only one or four pixels~\cite{tran2018spectral}.
On the other hand, Chen et al.~\cite{chen2017targeted} shows that by blending the whole instance with a certain pattern such as a fixed random noise pattern, it is possible to generate effective backdoors to poison the ML models.

In this work, we focus on certifying the robustness against general backdoor attacks, where the attacker is able to add any specific or uncontrollable random backdoor patterns for arbitrary adversarial targets.

\subsection{Randomized smoothing}

To defend against \textit{evasion attacks}, different approaches have been studied: some provide empirical approaches such as adversarial training~\cite{liu2018towards,cao2017mitigating}, and some provide theoretical guarantees against $L_p$ bounded adversarial perturbations. In particular, Cohen et al.~\cite{cohen2019certified} have proposed \textit{randomized smoothing} to certify the robustness of ML models against the $L_2$ norm bounded evasion attacks.

On a high level, the randomized smoothing technique~\cite{cohen2019certified} provides a way to certify the robustness of a \emph{smoothed} classifier against adversarial examples during test time.
First, a given classifier is smoothed by adding Gaussian noise $\epsilon \sim \mathcal{N}(0,\sigma^2\Id_d)$ around each test instance. Then, the classification gap between a lower bound of the confidence on the top-1 class $p_A$ and an upper bound of the confidence on the top-2 class $p_B$ are obtained. The smoothed classifier will be guaranteed to provide consistent predictions within the perturbation radius, which is a function of the standard deviation $\sigma$ of the smoothing noise, and the gap between the class probabilities $p_A$ and $p_B$, for each test instance.

However, all these approaches focus on the robustness against \textit{evasion attacks} only. In contrast, in this work, we aim to provide a function smoothing framework to certify the robustness against both evasion and poisoning attacks.
In particular, the current randomized smoothing strategy focuses on adding noise to induce smoothness on the level of {\em test instance}, while our unified framework generalizes this to smoothing on
the level of \textit{classifiers}.
Putting this generalization into practice in the context of certifying robustness against backdoor attacks naturally bears additional challenges which we describe and address in detail. In addition, we provide theoretical robustness guarantees for different machine learning models, smoothing noise distributions, as well as the tightness of the robustness bounds.

\begin{figure}
\centering
\includegraphics[width=0.49\textwidth]{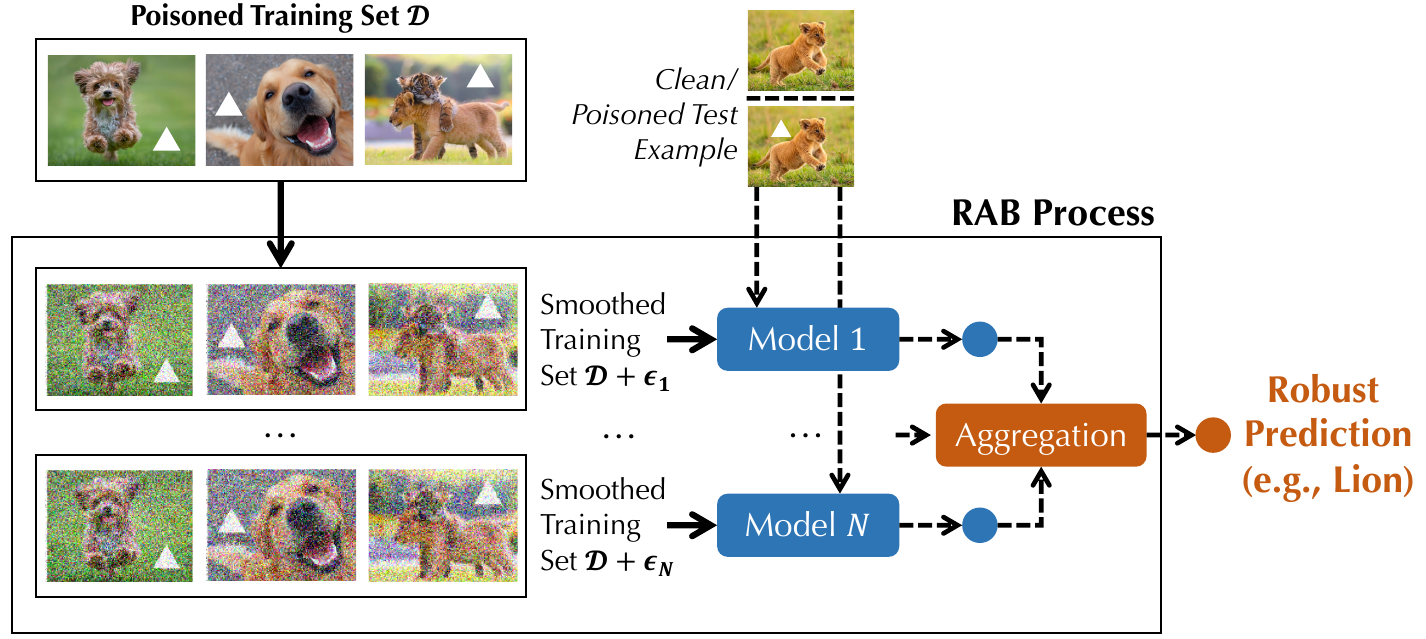}
\caption{An illustration of the RAB robust training process.
Given a poisoned training set $\cD + \Delta$ and a
training process $\mathcal{A}$ vulnerable to backdoor attacks, RAB
generates $N$ smoothed training sets $\{\cD_i\}_{i\in[N]}$
and trains $N$ different classifiers $\mathcal{A}_{i}$.}
\label{fig:process}
\end{figure}

\section{Threat Model and Method Overview}
\label{s:threat}

Here we first define the threat model including concise definitions of a backdoor attack,
and then introduce the method overview, where we define our robustness guarantee.

\subsection{Notation}

We write random variables as uppercase letters $X$ and use the notation $\Prob_X$ to denote the probability measure induced by $X$ and write $f_X$ to denote the probability density function. Realizations of random variables are written in lowercase letters.
For discrete random variables, we use lowercase letters to denote their probability mass function, e.g. $p(y)$ for distribution over labels.
Feature vectors are taken to be $d$-dimensional real vectors $x\in\bR^d$ and the set of labels $y$ for a $C$-multiclass classification problem is given by $\cC=\{1,\,\ldots,\,C\}$. A training set $\cD$ consists of $n$ (feature, label)-pairs $\cD = \{(x_1,\,y_1),\,\ldots,\,(x_n,\,y_n)\}$.
For a dataset $\cD$ and a collection of $n$ feature vectors $d = \{d_1,\,\ldots,\,d_n\}$, we write $\cD + d$ to denote the set $\{(x_1 + d_1,\,y_1),\,\ldots,\,(x_n + d_n,\,y_n)\}$.
We view a classifier as a deterministic function that takes as input a tuple with a test instance $x$ and training set $\cD$ and returns a class label $y\in\cC$. Formally, given a dataset $\cD$ and a test instance $x$, a classifier $h$ learns a conditional probability distribution $p(y\lvert\,x,\,\cD)$ over class labels and outputs the label which is deemed most likely under the learned distribution $p$:
\begin{equation}
    h(x,\,\cD) = \arg\max_y p(y\lvert\,x,\,\cD).
\end{equation}
We omit the dependence on model parameters throughout this paper and tacitly assume that the model is optimized based on training dataset $\cD$ via some optimization schemes such as stochastic gradient descent.

\subsection{Threat Model and the Goal of Defense}

\subsubsection{Threat Model} An adversary carries out a backdoor attack against a classifier $h$ and
a clean dataset $\mathcal{D}=\{(x_i, y_i)\}$.
The attacker has in mind a
target backdoor pattern $\Omega_x$ and
a target class $\tilde{y}$ and
the adversarial goal is to alter the dataset
such that, given a clean test example
$x$, adding the backdoor pattern to $x$ (i.e., $x + \Omega_x$)
will \textit{alter} the classifier output
$\tilde{y}$ with high probability.
In general, the attack can replace $r$ training instances $(x_i,\,y_i)$ by backdoored instances $(x_i+\Omega_x,\,\tilde{y}_i)$.
We remark that the attacker could embed
distinct patterns to each instance and our result
naturally extends to this case. Thus, summarizing the backdoor patterns as the collection $\Delta(\Omega_x) := \{\delta_1,\,\ldots,\,\delta_r,\,0,\ldots,0\}$, we formalize
a backdoor attack as the transformation $(\cD, \Omega_x, \tilde{y}) \to \cD_{BD}(\Omega_x, \tilde{y})$ with
\begin{equation}
    \begin{aligned}
        \cD_{BD}(\Omega_x, \tilde{y})
        =\{(x_i + \delta_i,\tilde{y}_i)\}_{i=1}^r \cup \{(x_i,y_i)\}_{i=r+1}^n
    \end{aligned}
\end{equation}
We often write $\cD_{BD}(\Omega_x)$
instead of $\cD_{BD}(\Omega_x, \tilde{y})$
when our focus is on the backdoor pattern
$\Omega_x$ instead of the target class
$\tilde{y}$.
The backdoor attack succeeds on
test example $x$ whenever
\begin{equation}
    h(x+\Omega_x, \cD_{BD}(\Omega_x)) = \tilde{y}
\end{equation}
\subsubsection{Goal of Defense}
\label{sssec:def-setting}
One natural goal to defend against
the above backdoor attack is to ensure that
the prediction of
$h(x+\Omega_x, \cD_{BD}(\Omega_x))$
is \textit{independent} of the backdoor patterns
$\Delta(\Omega_x)$ which are present in the dataset, i.e.,
\begin{equation}
    h(x+\Omega_x, \cD_{BD}(\Omega_x)) = h(x+\Omega_x, \cD_{BD}(\emptyset))
\end{equation}
where $\cD_{BD}(\emptyset)$
is the dataset without any embedded backdoor patterns
($\delta_i = 0$).
When this is true, the attacker
obtained \textit{no additional information}
by knowing the pattern $\Omega_x$
embedded in the training set. That is to say, given a test instance which may contain a backdoor pattern, its prediction stays the same, independent of whether the models were trained with or without backdoors. We assume that the defender has full control of the training process. See Section~\ref{sec:limitations} for more discussions on the assumptions and limitations of RAB.

\subsection{Method Overview}
\label{sec:method-overview}

\subsubsection{Certified Robustness against Backdoor Attacks}
We aim to obtain robustness bound $R$ such that, whenever the sum of the magnitude of backdoors is below $R$, the prediction of the backdoored classifier is the same as when the classifier is trained on benign data.
Formally, if $\cD_{BD}(\Omega_x)$ denotes the backdoored training set, and $\cD$ the training set containing clean features, we say that a classifier is \textit{provably robust} whenever
$\sqrt{\sum_{i=1}^r\left\|\delta_i\right\|_2^2} < R$
implies that $h(x+\Omega_x,\,\cD_{BD}(\Omega_x)) = h(x+\Omega_x,\,\cD_{BD}(\emptyset))$.

Our approach to obtaining the aforementioned robustness guarantee is based on randomized smoothing, which leads to the robust RAB training pipeline, as is illustrated in Figure~\ref{fig:process}.
Given a clean dataset $\cD$ and a backdoored dataset $\cD_{BD}(\Omega_x)$, the goal of the defender is to make sure that the prediction on test instances embedded with the pattern $\Omega_x$ is the same
as for models trained with
$\cD_{BD}(\emptyset)$.

Different from randomized smoothing-based certification against evasion attacks, here it is not enough to only smooth the test instances. Instead, in RAB, we will first add noise vectors, sampled from a smoothing distribution, to the given training instances, to obtain a collection of ``smoothed" training sets. We subsequently train a model on each training set and aggregate their final outputs together as the final ``smoothed" prediction. After this process, we show that it is possible to leverage the Neyman Pearson lemma to derive a robustness condition for this smoothed RAB training process. Additionally, the connection with the Neyman Pearson lemma also allows us to prove that the robustness bound is tight.
Note that the RAB framework requires the training instances to be ``smoothed" by a set of independent noises drawn from a certain distribution.

\vspace{0.5em}
\noindent{\bf Additional Challenges.}
We remark that, within this RAB training and certification process, there are several additional challenges. First, after adding noise to the training data, the clean accuracy of the trained classifier typically drops due to the distribution shift in the training data.
To mitigate this problem, we add a deterministic value, based on the hash of the trained model, to test instances (Section~\ref{sec:models}), which minimizes the distribution shift and leads to improved accuracy scores.
Second, considering different smoothing distributions for the training data, we provide rigorous analysis and a robustness bound for both Gaussian and uniform smoothing distributions (Section~\ref{sec:backdoor_attacks}).
Third, we note that the proposed training process requires sampling a large number of randomly perturbed training sets. As this is computationally expensive, we propose an efficient PTIME algorithm for $K$-NN classifiers (Section~\ref{sec:models}).

\vspace{0.5em}
\noindent{\bf Outline.}
In the following, we illustrate the RAB pipeline in three steps. In Section~\ref{sec:framework}, we introduce the theoretical foundations for a unified framework for certifying robustness against both evasion and backdoor attacks. In Section~\ref{sec:backdoor_attacks}, we introduce how to apply our unified framework to defend against backdoor attacks. In Section~\ref{sec:models}, we present RAB pipeline for two types of models --- DNNs and $K$-NN.

\section{Unified Framework for Certified Robustness}
\label{sec:framework}

In this section, we propose a unified theoretical framework for certified robustness against evasion and poisoning attacks for classification models. Our framework is based on the intuition that randomizing the prediction or training process will ``smoothen" the final prediction and therefore reduce the vulnerability to adversarial attacks. This principle has been successfully applied to certifying robustness against evasion attacks for classification models~\cite{cohen2019certified}. We first formally define the notion of a smoothed classifier where we extend upon previous work by randomizing \emph{both} the test instance and the training set. We then introduce basic terminology of hypothesis testing, from where we leverage the Neyman Pearson lemma to derive a generic robustness condition in Theorem~\ref{thm:main}. Finally, we show that this robustness condition is tight.

\subsection{Preliminaries}

\subsubsection{Smoothed Classifiers}
On a high level, a smoothed classifier $g$ is derived from a base classifier $h$ by introducing additive noise to the input consisting of test  and training instances. In a nutshell, the intuition behind randomized smoothing classifiers is that noise reduces the occurrence of regions with high curvature in the decision boundaries, resulting in reduced vulnerability to adversarial attacks. Recall that a classifier $h$, here serving as a base classifier, is defined as $h(x,\,\cD) = \arg\max_y p(y\lvert\,x,\,\cD)$ where $p$ is learned from a dataset $\cD$ and defines a conditional probability distribution over labels $y$. The final prediction is given by the most likely class under this learned distribution. A smoothed classifier is defined by
\begin{equation}
    \label{eq:smoothed_classifier}
    q(y\lvert\,x,\,\cD) = \Prob_{X,D}\left(h(x + X,\,\cD + D) = y\right)
\end{equation}
where we have introduced random variables $X\sim\Prob_X$ and $D\sim\Prob_D$ which act as smoothing distributions and  are assumed to be independent. We emphasize that $D$ is a collection of $n$ independent and identically distributed random variables $D^{(i)}$, each of which is added to a training instance in $\cD$. The final, smoothed classifier then assigns the most likely class to an instance $x$ under this new, ``smoothed" model $q$, so that
\begin{equation}
    g(x,\,\cD) = \arg\max_y q(y\lvert\,x,\,\cD).
\end{equation}
Within this formulation of a smoothed classifier, we can also model randomized smoothing for defending against evasion attacks by setting the training set noise to be zero, i.e. $D \equiv 0$. We emphasize at this point that the smoothed classifier $g$ implicitly depends on the choice of noise distributions $\Prob_X$ and $\Prob_D$. In section~\ref{sec:backdoor_attacks} we instantiate this classifier with Gaussian noise and with uniform noise and show how this leads to different robustness bounds.

\subsubsection{Statistical Hypothesis Testing}
Hypothesis testing is a statistical problem that is concerned with the question of whether or not some hypothesis that has been formulated is correct.
A decision procedure for such a problem is called a statistical hypothesis test.
Formally, the decision is based on the value of a realization $x$ for a random variable $X$ whose distribution is known to be either $\Prob_0$ (the null hypothesis) or $\Prob_1$ (the alternative hypothesis).
Given a sample $x\in\cX$, a randomized test $\phi$ can be modeled as a function $\phi\colon\cX \to [0,\,1]$ which rejects the null hypothesis with probability $\phi(x)$ and accepts it with probability $1 - \phi(x)$.
The two central quantities of interest are the probabilities of making a type I error, denoted by $\alpha(\phi;\,\Prob_0)$ and the probability of making a type II error, denoted by $\beta(\phi;\,\Prob_1)$. The former corresponds to the situation where the test $\phi$ decides for the alternative when the null is true, while the latter occurs when the alternative is true but the test decides for the null. Formally, $\alpha$ and $\beta$ are defined as
\begin{equation}
    \alpha(\phi;\,\Prob_0) = \bE_0(\phi(X)),\hspace{1em} \beta(\phi;\,\Prob_1) = \bE_1(1-\phi(X))
\end{equation}
where $\bE_{0}(\cdot)$ $(\bE_1(\cdot))$ denotes the expected value with respect to $\Prob_{0}$ $(\Prob_1)$.
The problem is to select the test $\phi$ which minimizes the probability of making a type II error, subject to the constraint that the probability of making a type-I error is below a given threshold $\alpha_0$. The Neyman Pearson lemma~\cite{neyman1933problem} states that a likelihood ratio test $\phi_{NP}$ is optimal, i.e. that $\alpha(\phi_{NP};\,\Prob_0) = \alpha_0$ and $\beta(\phi_{NP};\,\Prob_1) = \beta^*(\alpha_0;\,\Prob_0,\,\Prob_1)$ where
\begin{equation}
    \beta^*(\alpha_0;\,\Prob_0,\,\Prob_1) = \inf_{\phi\colon\alpha(\phi;\,\Prob_0) \leq \alpha_0} \beta(\phi;\,\Prob_1).
\end{equation}
In Theorem~\ref{thm:main}, we will see that we can leverage this formalism to get a robustness guarantee for smoothed classifiers. Additionally, stemming from the optimality of the likelihood ratio test, we show in Theorem~\ref{thm:tightness} that this condition is tight.

\subsection{A General Condition for Provable Robustness}

In this section, we derive a tight robustness condition by drawing a connection between statistical hypothesis testing and the robustness of classification models subject to adversarial attacks. We allow adversaries to conduct an attack on either
\emph{(i) the test instance $x$},
\emph{(ii) the training set $\cD$}
or \emph{(iii) a combined attack on test and training set}.
The resulting robustness condition is of a general nature and is expressed in terms of the optimal type II errors for likelihood ratio tests.
We remark that this theorem is a more general version of the result presented in~\cite{cohen2019certified}, by extending it to general smoothing distributions and smoothing on the training set.
In Section~\ref{sec:backdoor_attacks} we will show how this result can be used to obtain robustness bound in terms of $L_p$-norm bounded backdoor attacks. We show that smoothing on the training set makes it possible for certifying the robustness against backdoors, and the general smoothing distribution allows us to explore the robustness bound certified by different smoothing distributions.

\begin{thm}
    \label{thm:main}
    Let $q$ be the smoothed classifier as in~\eqref{eq:smoothed_classifier} with smoothing distribution $Z:=(X,\,D)$ with $X$ taking values in $\bR^d$ and $D$ being a collection of $n$ independent $\bR^d$-valued random variables, $D=(D^{(1)},\,\ldots,\,D^{(n)})$. Let $\Omega_x\in\bR^d$ and let $\Delta:=(\delta_1,\,\ldots,\,\delta_n)$ for backdoor patterns $\delta_i\in\bR^d$. Let $y_A \in \cC$ and let $p_A,\,p_B\in[0,\,1]$ such that $y_A = g(x,\,\cD)$ and
    \begin{equation}
        \label{eq:confidence}
        q(y_A\lvert\,x,\,\cD) \geq p_A > p_B \geq \max_{y\neq y_A} q(y\lvert\,x,\,\cD).
    \end{equation}
    If the optimal type II errors, for testing the null $Z\sim\Prob_0$ against the alternative $Z + (\Omega_x,\,\Delta) \sim \Prob_1$, satisfy
    \begin{equation}
        \label{eq:robustness_condition}
        \beta^*(1-p_A;\,\Prob_0,\,\Prob_1) + \beta^*(p_B;\,\Prob_0,\,\Prob_1) > 1,
    \end{equation}
    then it is guaranteed that $y_A = \arg\max_y q(y\lvert\,x + \Omega_x,\,\cD + \Delta)$.
\end{thm}
The following is a short sketch of the proof for this theorem. We refer the reader to Appendix~\ref{appendix:main_theorem_proof} for details.
\begin{proof}[Proof (Sketch)]
    We first explicitly construct the likelihood ratio tests $\phi_A$ and $\phi_B$ for testing the null hypothesis $Z$ against the alternative $Z + (\Omega_x,\,\Delta)$ with type I errors $\alpha(\phi_A;\,\Prob_0) = 1-p_A$ and $\alpha(\phi_B;\,\Prob_0) = p_B$ respectively.
    An argument similar to the Neyman-Pearson Lemma~\cite{neyman1933problem} shows that the class probability for $y_A$ given by $q$ on the perturbed input is lower bounded by $\beta(\phi_A;\,\Prob_1) = \beta^*(1-p_A;\,\Prob_0,\,\Prob_1)$.
    A similar reasoning leads to the fact that an upper bound on the prediction score for $y\neq y_A$ on the perturbed input is given by $1-\beta(\phi_B;\,\Prob_1) = 1- \beta^*(p_B;\,\Prob_0,\,\Prob_1)$.
    Combining this leads to condition~(\ref{eq:robustness_condition}).
\end{proof}

We now make some \underline{observations} about Theorem~\ref{thm:main} to get intuition on the robustness condition~(\ref{eq:robustness_condition}):
\begin{itemize}[leftmargin=*]
    \item Different smoothing distributions lead to robustness bounds in terms of different norms. For example, Gaussian noise yields robustness bound in $L_2$ norm while Uniform noise leads to other $L_p$ norms.
    \item The robustness condition~(\ref{eq:robustness_condition}) does not make any assumption on the underlying classifier other than on the class probabilities predicted by its smoothed version.
    \item The random variable $Z + (\Omega_x,\,\Delta)$ models a general adversarial attack including evasion and backdoor attacks.
    \item If no attack is present, i.e., if $(\Omega_x,\,\Delta)=(0,\,0)$, then we get the trivial condition $p_A > p_B$.
    \item As $p_A$ increases, the optimal type II error increases for given backdoor $(\Omega_x,\,\Delta)$. Thus, in the simplified setup where $p_A + p_B = 1$ and the robustness condition reads $\beta^*(1-p_A;\,\Prob_0,\,\Prob_1) > 1/2$,  the distribution shift caused by $(\Omega_x,\,\Delta)$ can increase. Thus, as the smoothed classifier becomes more confident, the robust region becomes larger.
\end{itemize}

While the generality of Theorem~\ref{thm:main} allows us to model a multitude of threat models, it bears the challenge of how one should instantiate this theorem such that it is applicable to defend against a specific adversarial attack.
In addition to the flexibility with regard to the underlying threat model, we are also provided with flexibility regarding the smoothing distributions, resulting in different robustness guarantees.
This again begs the question, of which smoothing distribution results in useful robustness bounds.
In Section~\ref{sec:backdoor_attacks}, we will show how this theorem can be applied to obtain the robustness guarantee against backdoor attacks described in Section~\ref{s:threat}.

Next, we show that our robustness condition is tight in the following sense:
If~(\ref{eq:confidence}) is all that is known about the smoothed classifier $g$, then there is no perturbation $(\Omega_x,\Delta)$ that violates~(\ref{eq:robustness_condition}). On the other hand, if~(\ref{eq:robustness_condition}) is violated, then we can always construct a smoothed classifier $g^*$ such that it satisfies the class probabilities~(\ref{eq:confidence}) but is not robust against this perturbation.
\begin{thm}
    \label{thm:tightness}
    Suppose that $1 \geq p_A + p_B \geq 1 - (C-2)\cdot p_B$. If the adversarial perturbations $(\Omega_x,\,\Delta)$ violate~(\ref{eq:robustness_condition}), then there exists a base classifier $h^*$ such that the smoothed classifer $g^*$ is consistent with the class probabilities~(\ref{eq:confidence}) and for which $g^*(x+\Omega_x,\,\cD+\Delta) \neq y_A$.
\end{thm}

\section{Provable Robustness Against Backdoors}
\label{sec:backdoor_attacks}

It is not straightforward to use the result from Theorem~\ref{thm:main} to get a robustness certificate against backdoor attacks in terms of $L_p$-norm bounded backdoor patterns.
In this section, we aim to answer the question: \emph{how can we instantiate this result to obtain robustness guarantees against backdoor attacks?}
In particular, we show that by leveraging Theorem~\ref{thm:main}, we obtain the robustness guarantee defined in Section~\ref{s:threat}.
To that end, we derive robustness bounds for smoothing with isotropic Gaussian noise
and we also illustrate how to
derive certification bounds
using other smoothing distributions.
Since isotropic Gaussian noise leads to a better radius, we will use this distribution in our experiments as a demonstration.

\subsection{Method Outline}

\subsubsection{Intuition}
Suppose that we are given a base classifier that has been trained on a \emph{backdoored} dataset that contains $r$ training samples which are infected with backdoor patterns $\Delta(\Omega_x)$.
Our goal is to derive a condition on the backdoor patterns $\Delta(\Omega_x)$ such that the prediction for $x + \Omega_x$ with a classifier trained on the backdoored dataset $\cD_{BD}(\Delta(\Omega_x))$ is the same as the prediction (on the same input) that a smoothed classifier would have made, had it been trained on a dataset without the backdoor triggers, $\cD_{BD}(\emptyset)$.
In other words, we obtain the guarantee that \emph{an attacker can not achieve their goal of systematically leading the test instance with the backdoor pattern to the adversarial target}, meaning they will always obtain the same prediction as long as the added pattern $\delta$ satisfies certain conditions (bounded magnitude).
\subsubsection{Gaussian Smoothing}
We obtain this certificate by instantiating Theorem~\ref{thm:main} in the following way.
Suppose an attacker injects backdoor patterns $\Delta(\Omega_x) = \{\delta_1,\,\ldots,\,\delta_r\}\subset\bR^d$ to $r \leq n$ training instances of the training set $\cD$, yielding the backdoored training set $\cD_{BD}(\Delta(\Omega_x))$.
We then train the base classifier on this poisoned dataset, augmented with additional noise on the feature vectors
$\cD_{BD}(\Delta(\Omega_x)) + D$, where $D$ is the smoothing noise added to the training features. We obtain a prediction of the smoothed classifier $g$ by taking the expectation with respect to the distribution of the smoothing noise $D$.
Suppose that the smoothed classifier obtained in this way predicts a malicious instance $x + \Omega_x$ to be of a certain class with probability at least $p_A$ and the runner-up class with probability at most $p_B$.
Our result tells us that, as long as the introduced patterns satisfy condition~(\ref{eq:robustness_condition}), we get the guarantee that the malicious test input would have been classified equally as when the classifier had been trained on the dataset with clean features $\cD_{BD}(\emptyset)$.
In the case where the noise variables are isotropic Gaussians with standard deviation $\sigma$, the condition~(\ref{eq:robustness_condition}) yields a robustness bound in terms of the sum of $L_2$-norms of the backdoors.
\begin{cor}[Gaussian Smoothing]
    \label{cor:gaussian}
    Let $\Delta = (\delta_1,\,\ldots,\,\delta_n)$ and $\Omega_x$ be $\bR^d$-valued backdoor patterns and let $\cD$ be a training set. Suppose that for each $i$, the smoothing noise on the training features is $D^{(i)} \overset{iid}{\sim} \cN(0,\,\sigma^2\Id_d)$. Let $y_A\in\cC$ such that $y_A = g(x+\Omega_x,\,\cD + \Delta)$ with class probabilities satisfying
    \begin{equation}
        \begin{split}
        q(y_A\lvert\,x+\Omega_x,\,\cD + \Delta) &\geq p_A\\&\hspace{-4em} > p_B \geq \max_{y\neq y_A} q(y\lvert\,x + \Omega_x,\,\cD + \Delta).
        \end{split}
    \end{equation}
    Then, if the backdoor patterns are bounded by
    \begin{equation}
        \label{eq:gaussian_bound}
        \sqrt{\sum_{i=1}^n \left\|\delta_i\right\|_2^2} < \frac{\sigma}{2}\left(\Phi^{-1}(p_A) - \Phi^{-1}(p_B)\right),
    \end{equation}
     it is guaranteed $y_A = g(x + \Omega_x,\, \cD) = g(x + \Omega_x,\, \cD + \Delta)$.
\end{cor}
This result shows that, whenever the norms of the backdoor patterns are below a certain value, we obtain the guarantee that the classifier makes the same prediction on the test data with backdoors  as it does when trained without embedded patterns in the training set.
We can further simplify the robustness bound in~(\ref{eq:gaussian_bound}) if we can assume that an attacker poisons at most $r\leq n$ training instances with one single pattern $\delta$. In this case, the bound~(\ref{eq:gaussian_bound}) is given by
\begin{equation}
    \label{eq:gaussian_bound_single}
    \left\|\delta\right\|_2 < \frac{\sigma}{2\sqrt{r}}\left(\Phi^{-1}(p_A) - \Phi^{-1}(p_B)\right).
\end{equation}
We see that, as we know more about the capabilities of an attacker and the nature of the backdoor patterns, we are able to certify a larger robustness radius, proportional to $1/\sqrt{r}$.

\subsection{Other Smoothing Distributions}
Given the generality of our
framework, it is possible to
derive certification
bounds using other
smoothing distributions.
However, different smoothing
distributions have vastly
different performance and
a comparative study
among different smoothing
distributions is
interesting future work.
In this paper, we will just illustrate
one example of smoothing
using a uniform distribution.

\begin{cor}[Uniform Smoothing]
    \label{cor:uniform}
    Let $\Delta = (\delta_1,\,\ldots,\,\delta_n)$ and $\Omega_x$ be $\bR^d$ valued backdoor patterns and let $\cD$ be a training set. Suppose that for each $i$, the smoothing noise on the training features is $D^{(i)} \overset{iid}{\sim} \cU([a,\,b])$. Let $y_A\in\cC$ such that $y_A = g(x+\Omega_x,\,\cD + \Delta)$ with class probabilities satisfying
    \begin{equation}
        \begin{split}
        q(y_A\lvert\,x+\Omega_x,\,\cD + \Delta) &\geq p_A
        \\&\hspace{-4em}
        > p_B \geq \max_{y\neq y_A} q(y\lvert\,x + \Omega_x,\,\cD+ \Delta).
        \end{split}
    \end{equation}
    Then, if the backdoor patterns satisfy
    \begin{equation}
        \label{eq:uniform_bound}
        1 - \left(\dfrac{p_A - p_B}{2}\right) < \prod_{i=1}^n\left(\prod_{j=1}^d\left(1 - \dfrac{|\delta_{i,j}|}{b-a}\right)_+\right)
    \end{equation}
    where $(x)_+ = \max\{x,0\}$, it is guaranteed that $y_A = g(x + \Omega_x,\, \cD) = g(x + \Omega_x,\, \cD + \Delta)$.
\end{cor}
As in the Gaussian case, the robustness bound in~\eqref{eq:uniform_bound} can again be simplified in a similar fashion, if we assume that an attacker poisons at most $r\leq n$ training instances with one single pattern $\delta$.
In this case, the bound~\eqref{eq:uniform_bound} is given by
\begin{equation}
    \label{eq:uniform_bound_simpler}
    1 - \left(\dfrac{p_A - p_B}{2}\right) < \left(\prod_{j=1}^d\left(1 - \frac{\left|\delta_{j}\right|}{b-a}\right)_+\right)^r.
\end{equation}
We see again that, as the number of infected training samples $r$ gets smaller, this corresponds to a larger bound since the RHS of~\eqref{eq:uniform_bound_simpler} gets larger. In other words, if we know that the attacker injects fewer backdoors, then we can certify a backdoor pattern with a larger magnitude.

\vspace{0.5em}
\noindent{\bf Discussions.}
We emphasize that in this paper, we focus on
protecting the system against
attackers who aim to \textit{trigger}
a targeted error with a specific
\textit{backdoor pattern}. The system can
still be vulnerable to other
types
of \textit{poisoning attacks}. One such example is
the label flipping attack,
in which
one flips the labels of a subset of
examples while keeping the features
untouched. Interestingly,
one concurrent work
explored the possibility of using
randomized smoothing to defend against
label flipping attack~\cite{rosenfeld2020certified}.
Developing a single framework to
be robust against both backdoor
and label flipping attacks is an exciting
future direction, and we expect it
to require nontrivial extensions
of both approaches
to achieve non-trivial certified accuracy.
Furthermore, while we focus the experiments on Gaussian smoothing and $L_2$-norm guarantees, it is in principle possible to certify other $L_p$-norms with different smoothing distributions.
For evasion attacks,~\cite{li2021tss} use exponential smoothing noise with certificates in $L_1$-norm.
Such analysis of different smoothing distributions for different experimental settings goes beyond the scope of this work and is interesting for future research.

\section{Instantiating the General Framework
with Specific ML Models}
\label{sec:models}

In the preceding sections, we presented our approach to certifying robustness against backdoor attacks.
Here, we will analyze and provide detailed algorithms for the RAB training pipeline for two types of machine learning models: deep neural networks and $K$-nearest neighbor classifiers.
The success of backdoor poisoning attacks against DNNs has caused a lot of attention recently.
Thus, we first aim to evaluate and certify the robustness of DNNs against backdoor attacks.
Secondly, given the fact that $K$-NN models have been widely applied in different applications, either based on raw data or on trained embeddings, it is of great interest to know about the robustness of this type of ML models. Specifically,
we are inspired by a recent
result~\cite{ccc} and develop an \textit{exact}
efficient smoothing algorithm for $K$-NN models, such that we do not need to draw a large number of random samples from the smoothing distribution for these models. This makes our approach considerably more practical for this type of classifier as it avoids the expensive training of a large number of models, as is required with generic classification algorithms including DNNs.

\begin{algorithm}[t!]
\caption{\textsc{DNN-RAB} for training certifiably robust DNNs.}
\label{alg:rab-train-dnn}
\begin{algorithmic}[1]
\REQUIRE Poisoned training dataset $\cD=\{(x_i + \delta_i, \tilde{y}_i)_{i=1}^n\}$, noise scale $\sigma$, model number $N$
\FOR{$k=1,\ldots, N$}
	\STATE Sample $\epsilon_{k,1}, \ldots, \, \epsilon_{k,n} \overset{\mathrm{iid}}{\sim} \cN(0, \sigma^2\Id_d)$.
	\STATE $\cD_k = \{(x_i + \delta_i + \epsilon_{k,i},\,\tilde{y}_i)_{i=1}^n\}$.
	\STATE $h_k = \texttt{train\_model}(\cD_k)$.
	\STATE Sample $u_k$ from $\mathcal{N}(0,\sigma^2{\Id_d})$ deterministically with random seed based on $hash(h_k)$.
\ENDFOR
\RETURN Model collection $\{(h_1,\,u_1),\,\ldots,\,(h_N,\,u_N)\}$
\end{algorithmic}
\end{algorithm}

\begin{algorithm}[t!]
\caption{Certified inference with RAB-trained models.}
\label{alg:rab-inference-dnn}
\begin{algorithmic}[1]
\REQUIRE Test sample $x$, noise scale $\sigma$, models $\{(h_k,\,u_k)\}_{k=1}^N$, backdoor magnitude $\|\delta\|_2$, number of poisoned training samples $r$
\STATE \texttt{counts} = $\abs{\{k\colon\,h_k(x + u_k,\,\cD + \epsilon_k) = y\}}$ for $y = 1,\,\ldots,\,C$
\STATE $y_A,\,y_B = $ top two indices in \texttt{counts}
\STATE $n_A,\,n_B =$ \texttt{counts}[$y_A$], \texttt{counts}[$y_B$]
\STATE $p_A, p_B = \texttt{calculate\_bound}(n_A, n_B, N, \alpha)$.
\IF {$p_A > p_B$}
    \STATE $R=\frac{\sigma}{2\sqrt{r}}\left(\Phi^{-1}(p_A) - \Phi^{-1}(p_B)\right)$
		\IF {$R \geq \|\delta\|_2$}
			\RETURN prediction $y_A$, robust radius $R$.
		\ENDIF
\ENDIF
\RETURN ABSTAIN
\end{algorithmic}
\end{algorithm}

\subsection{Deep Neural Networks}

\label{sec:models-dnn}
In this section, we consider smoothed models which use DNNs as base classifiers.
For a given test input $x_{test}$, the goal is to calculate the prediction of $g$ on $(x_{test}, \cD + \Delta)$ according to Corollary~\ref{cor:gaussian} and the corresponding certified bound given in the right hand side of Eq.~\eqref{eq:gaussian_bound}.
In the following, we first describe the training process and then the inference algorithm.

\subsubsection{RAB Training for DNNs}
First, we draw $N$ samples $d_1,\,\ldots,\,d_N$ from the distribution of $D\sim\prod_{i=1}^n\cN(0,\,\sigma^2\Id_d)$.
Given the $N$ samples of training noise (each consisting of $\abs{\cD} = n$ noise vectors), we train $N$ DNN models on the datasets $\cD+d_k$ for $k=1,\,\ldots,\,N$ and obtain classifiers $h_1,\,\ldots,\,h_N$. Along with each model $h_k$, we draw a random noise $u_k$ from $\cN(0,\,\sigma^2\Id_d)$ with a random seed based on the hash of the trained model file. This noise vector is stored along with the model parameters and added to each test input during inference. The reason for this is that, empirically, we observed that inputting test samples without this additional augmentation leads to poor prediction performance since the ensemble of models $\{h_1,\,\ldots,\,h_N\}$ has to classify an input that has not been perturbed by Gaussian noise, while it has only ``seen“ noisy samples, leading to a mismatch between training and test distributions.
Algorithm~\ref{alg:rab-train-dnn} shows the pseudocode describing RAB-training for DNN models.

\subsubsection{Inference}
To get the prediction of the smoothed classifier on a test sample $x_{test}$ we first compute the empirical majority vote as an unbiased estimate
\begin{equation}
    \label{eq:smoothed_dnn_1}
    \hat{q}(y\lvert\,x,\,\cD) = \frac{\abs{\left\{k\colon\,h_k(x_{test} + u_k,\,\cD + d_k)=y\right\}}}{N}
\end{equation}
of the class probabilities and where $u_k$ is the (model-) deterministic noise vector sampled during training in Algorithm~\ref{alg:rab-train-dnn}.
Second, for a given error tolerance $\alpha$, we compute $p_A$ and $p_B$ using one-sided $(1-\alpha)$ lower confidence intervals for the binomial distribution with parameters $n_A$ and $n_B$ and $N$ samples. Finally, we invoke Corollary~\ref{cor:gaussian} and first compute the robust radius according to Eq.~\eqref{eq:gaussian_bound_single}, based on $p_A,\,p_B$ the smoothing noise parameter $\sigma$ and the number of poisoned training samples $r$. If the resulting radius $R$ is larger than the magnitude of the backdoor samples $\delta$, the prediction is certified, i.e. the backdoor attack has failed on this particular sample. Algorithm~\ref{alg:rab-inference-dnn} shows the pseudocode for the DNN inference with RAB.

\subsubsection{Model-deterministic Test-time Augmentation}
\label{subsubsec:hash-aug}
One caveat in directly applying Equation~(\ref{eq:smoothed_dnn_1})
is the mismatch of the
training and test distribution --- during training, all examples are perturbed
with sampled noise, whereas the test
example is without noise. In practice, we see that this mismatch significantly decreases
the test accuracy. One natural
idea is to also add
noise to the test examples, however, this requires careful design (e.g., simply drawing $k$ independent noise vectors and applying them to Equation~(\ref{eq:smoothed_dnn_1}) will lead to a less powerful bound).
We thus modify the inference function
given a learned model $h_k$ in the following way. Instead of directly classifying an unperturbed input $x_{test}$, we use the hash value of the trained $h_k$ model parameters as the random seed
and sample $u_k \sim \mathcal{N}_{hash(h_k)}(0, \sigma^2{\Id_d})$. In practice, we use SHA256 hashing\cite{wiki:sha256} of the trained model file. In this way,
the noise we add is a deterministic
function of the trained model,
which is equivalent to altering
the inference function in a deterministic way, $\Tilde{h}_k(x_{test}) = h_k(x_{test} + u_k)$. We
show in the experiments that this leads to significantly
better prediction performance in practice.
Note that the reason for using a hash function instead of random sampling every time is to ensure that the noise generation process is deterministic, so the choice of different hash functions is flexible.

\subsection{K-Nearest Neighbors}

If the base classifier $h$ is a $K$-nearest neighbor classifier, we can evaluate the corresponding smoothed classifier \textit{exactly} and efficiently, in polynomial time, if the smoothing noise is drawn from a Gaussian distribution. In other words, for this type of model, we can eliminate the need to approximate the expectation value via Monte Carlo sampling and evaluate the classifier exactly. Finally, it is worth remarking that bypassing the need to do Monte Carlo sampling ultimately results in a considerable speed-up as it avoids the expensive training of independent models as is required for generic models including DNNs.

A $K$-NN classifier works in the following way:
Given a training set $\cD=\{(x_i,\,y_i)_{i=1}^n\}$ and a test example
$x$, we first calculate the similarity between $x$ and each $x_i$,
$s_i:=\kappa(x_i,\,x)$ where $\kappa$ is a similarity function. Given all these
similarity scores $\{s_i\}_i$, we choose the $K$ most similar training
examples with the largest similarity score $\{x_{\sigma_i}\}_{i=1}^K$ along with
corresponding labels $\{y_{\sigma_i}\}_{i=1}^K$. The final prediction is made
according to a majority vote among the top-$K$ labels.

Similar to DNNs, we obtain a smoothed $K$-NN classifier
by adding Gaussian noise to training points and evaluate the
expectation with respect to this noise distribution
\begin{equation}
    \label{eq:smoothed_knn_classifier}
    q(y\lvert\,x,\,\cD) = \Prob\left(\text{$K$-NN}(x,\,\cD + D) = y\right)
\end{equation}
where $D=(D^{(1)},\,\ldots,\,D^{(n)})\sim\prod_{i=1}^n\cN(0,\,\sigma^2\Id_d)$. The next theorem shows that~(\ref{eq:smoothed_knn_classifier}) can be computed
exactly and efficiently when we measure the similarity with euclidean distance quantized into finite number similarity of levels.
\begin{thm}
    \label{thm:knn_complexity}
    Given $n$ training instances, a $C$-multiclass $K$-NN classifier based on quantized euclidean distance with $L$ similarity levels, smoothed with isotropic Gaussian noise can be evaluated exactly with complexity $\cO(K^{2+C}\cdot n^2 \cdot L \cdot C)$.
\end{thm}
\begin{proof}[Proof (sketch)]
    The first step to computing~(\ref{eq:smoothed_knn_classifier}) is to notice that we can summarize all possible arrangements $\{x_{\sigma_i} + D^{(\sigma_i)}\}_{i=1}^K$ of top-$K$ instances leading to a prediction by using tally vectors $\gamma\in[K]^C$. A tally vector has as its $k$-th element the number of instances in the top-$K$ with label $k$, $\gamma_k = \#\{y_{\sigma_i}=k\}$. In the second step, we partition the event that a tally vector $\gamma$ occurs into events where an instance $i$ with similarity $\beta$ is in the top-$K$ but would not be in the top-$(K-1)$. These first two steps result in a summation over $\cO(K^{C}\cdot n \cdot L\cdot C)$ terms. In the last step, we compute the probabilities of the events $\{\mathrm{tally}\,\gamma\,\land\,\kappa(x_i+ D^{(i)},\,x)=\beta\}$ with dynamic programming in $\cO(n\cdot K^2)$ steps, resulting in a final time complexity of $\cO(K^{2+C}\cdot n^2 \cdot L \cdot C)$.
\end{proof}
If $K=1$, an efficient algorithm can even achieve time complexity linear in the number of training samples $n$.
We refer the reader to Appendix~\ref{appendix:knn}
for details and the algorithm.

\section{Experimental Results}
\label{sec:experiments}

In this section, we present an extensive experimental evaluation of our approach and provide a benchmark for certified robustness for DNN and KNN classifiers on different datasets. In addition, we consider three different types of backdoor attack patterns, namely one-pixel, four-pixel, and blending-based attacks.
The attack patterns are illustrated in Figure~\ref{fig:backdoor-examples} which shows that these patterns can be hard to spot by a human, in particular for the one-pixel pattern on high-resolution images.
At a high level,
our experiments reveal the following set of observations: \footnote{Our code is available at \url{https://github.com/AI-secure/Robustness-Against-Backdoor-Attacks}}

\begin{itemize}
\item  RAB is able to achieve comparable robustness on benign instances compared with vanilla trained models, and achieves
non-trivial \textit{certified accuracy}
 under a range of realistic backdoor attack settings.
\item There is a gap between
the certified accuracy provided by RAB and
empirical robust accuracy achieved by
the state-of-the-art empirical defenses against backdoor attacks without
formal guarantees, which serves as the upper bound of the certified accuracy; however,
such a gap is reasonably
small and we are optimistic
that future research can
further close this gap.
\item RAB's efficient KNN
algorithm provides
a very effective solution for
tabular data.
\item Simply applying
randomized smoothing to
RAB is not effective and
careful optimizations (e.g., deterministic test-time augmentation)
are necessary.
\end{itemize}

\begin{figure}[t!]
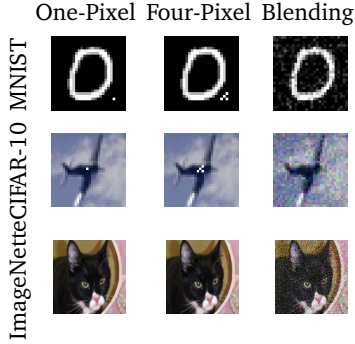

		\centering
		\adjustboxset{height=.05\textheight, valign=c,  margin=0pt 1pt 0pt 1pt}
		\setlength{\tabcolsep}{1pt}
		\begin{tabular}{c c c c}
			& \small One-Pixel & \small Four-Pixel & \small Blending\\
			\parbox[t]{2mm}{\rotatebox[origin=c]{90}{\small MNIST}} & \adjustimage{}{arxiv-figures/trojaneg-mnist-onepixel}& \adjustimage{}{arxiv-figures/trojaneg-mnist-fourpixel}& \adjustimage{}{arxiv-figures/trojaneg-mnist-blending}\\
			\parbox[t]{2mm}{\rotatebox[origin=c]{90}{\small CIFAR-10}} & \adjustimage{}{arxiv-figures/trojaneg-cifar-1p3c-add}& \adjustimage{}{arxiv-figures/trojaneg-cifar-4p3c-add}& \adjustimage{}{arxiv-figures/trojaneg-cifar-blending}\\
			\parbox[t]{2mm}{\rotatebox[origin=c]{90}{\small ImageNette}} & \adjustimage{}{arxiv-figures/trojaneg-dogcat-1p3c-add}& \adjustimage{}{arxiv-figures/trojaneg-dogcat-4p3c-add}& \adjustimage{}{arxiv-figures/trojaneg-dogcat-blending}\\
		\end{tabular}
	\caption{Examples of the applied backdoor patterns.}
	\label{fig:backdoor-examples}
\end{figure}

\begin{table*}
	\caption{Evaluation on \textbf{DNNs} with different datasets. We use $\sigma=0.5$ for MNIST and $\sigma=0.2$ for CIFAR-10 and ImageNette. ``Vanilla" denotes DNNs without RAB training and ``RAB-cert” is the certified accuracy of RAB. The highest empirical robust accuracies are {\bf bolded}. The robust accuracy scores are evaluated only on \emph{successfully backdoored instances}. 
	}
	\label{tab:result-dnn}
	\centering
	\resizebox{\linewidth}{!}{
	\begin{tabular}{l c c c c c c c c c c c c c}
	\toprule
    & \multirow{2}{*}[-2pt]{\makecell{Backdoor \\ Pattern}} & \multicolumn{2}{c}{Acc. on Benign Instances} & \multicolumn{8}{c}{Empirical Robust Acc.} & \textbf{Certified Robust Acc}.\\
    \cmidrule(lr){3-4}\cmidrule(lr){5-12}\cmidrule(lr){13-13}
    & & Vanilla & RAB & Vanilla & RAB & AC~\cite{chen2018detecting} & Spectral~\cite{tran2018spectral} & Sphere~\cite{steinhardt2017certified} & NC~\cite{wang2019neural} & SCAn~\cite{tang2021demon} & Mixup~\cite{borgnia2021strong} & RAB-cert\\
    \midrule
    \multirow{3}{*}{MNIST} & One-pixel & 92.7\% & 92.6\% & 0\% & 41.2\% & 64.3\% & 3.4\% & 3.1\% & {\bf 76.2\%} & 45.6\% & 34.5\% & 23.5\% \\
    & Four-pixel & 92.7\% & 92.6\% & 0\% & 40.7\% & 56.9\% & 2.8\% & 2.1\% & {\bf79.9\%} & 45.4\% & 33.2\% & 24.1\% \\
    & Blending & 92.9\% & 92.6\% & 0\% & 39.6\% & {\bf63.6\%} & 3.0\% & 1.8\% & 63.0\% & 44.7\% & 28.3\% & 23.1\% \\
    \midrule
    \multirow{3}{*}{CIFAR-10} & One-pixel & 59.9\% & 56.7\% & 0\% & {\bf42.9\%} & 31.4\% & 31.2\% & 16.5\% & 15.7\% & 12.9\% & 26.5\% & 24.5\% \\
    & Four-pixel & 59.4\% & 56.8\% & 0\% & {\bf42.8\%} & 28.9\% & 31.4\% & 15.0\% & 16.8\% & 16.5\% & 31.8\% & 24.1\% \\
    & Blending & 60.5\% & 56.8\% & 0\% & {\bf42.8\%} & 27.4\% & 28.0\% & 16.5\% & 16.6\% & 15.8\% & 30.0\% & 24.1\% \\
    \midrule
    \multirow{3}{*}{ImageNette} & One-pixel & 93.0\% & 91.6\% & 0\% & 38.6\% & 44.7\% & 47.8\% & 29.6\% & {\bf69.9\%} & 35.2\% & 55.1\% & 15.9\%\\
    & Four-pixel & 93.7\% & 91.5\% & 0\% & 38.4\% & 54.2\% & 52.8\% & 42.1\% & {\bf67.9\%} & 49.7\% & 51.6\% & 12.6\% \\
    & Blending & 94.8\% & 91.8\% & 0\% & 29.9\% & 46.3\% & 18.4\% & 31.0\% & {\bf 66.7\%} & 33.3\% & 56.3\% & 9.2\% \\
	\bottomrule
	\end{tabular}
	}
\end{table*}
\begin{table*}
	\caption{Evaluation on \textbf{KNNs} with $K=3$ on the  UCI Spambase \textbf{tabular dataset}. We use $\sigma=0.5$ for Spam. ``Vanilla" denotes DNNs without RAB training and ``RAB-cert” is the certified accuracy of RAB. The highest empirical robust accuracies are {\bf bolded}. The robust accuracy scores are evaluated only on \emph{successfully backdoored instances}.}
	\label{tab:result-knn}
	\centering
	\resizebox{\linewidth}{!}{
	\vspace{-0.5em}
	\begin{tabular}{l c c c c c c c c c c c}
		\toprule
		& \multirow{2}{*}[-2pt]{\makecell{Backdoor \\ Pattern}} & \multicolumn{2}{c}{
		Accuracy on Benign Instances} & \multicolumn{6}{c}{
		Empirical Robust Acc.} & \textbf{Certified Robust Acc}.\\
        \cmidrule(lr){3-4}\cmidrule(lr){5-10}\cmidrule(lr){11-11}
		& & Vanilla & RAB & Vanilla & RAB & AC~\cite{chen2018detecting} & Spectral~\cite{tran2018spectral} & Sphere~\cite{steinhardt2017certified} & SCAn~\cite{tang2021demon} & RAB-cert\\
		\midrule
		\multirow{3}{*}{UCI Spambase} & One-pixel & 98.7\% & 98.4\% & 0\% & {\bf54.6\%} & 9.0\% & 9.6\% & 2.4\% & 10.5\% & 36.4\% \\
		& Four-pixel & 98.7\% & 98.4\% & 0\% & {\bf50.0\%} & 9.6\% & 9.6\% & 3.0\% & 11.2\%  & 33.3\%\\
		& Blending & 98.7\% & 98.4\% & 0\% & {\bf58.3\%} & 8.1\% & 8.1\% & 1.7\% & 9.9\% & 41.7\%\\
		\bottomrule
	\end{tabular}
	}
	\vspace{-1.5em}
\end{table*}

\subsection{Experiment Setup}

\label{subsec:experiment-setup}
In this paper, we follow the popular transfer learning setting for poisoning attacks~\cite{shafahi2018poisonfrogs,chaudhuri2019transferable,saha2020hidden,gu2017badnets,meila2021just} in our experiments, specifically ~\cite{schwarzschild2021just}.
We first use models initialized with pretrained weights obtained from a clean dataset, and then finetune the model with a subset of training data containing backdoored instances.
Preliminary experiments and existing work~\cite{wang2018defending} showed that it is difficult to successfully inject backdoors if only a subset of parameters is finetuned. As a result, we always finetune the entire set of model parameters.

\subsubsection{Datasets and Model}
We consider four different datasets, namely the MNIST dataset~\cite{lecun1998gradient} consisting of 60,000 images of handwritten digits from $0$-$9$, the CIFAR-10 dataset\cite{krizhevsky2009learning} which includes 50,000 images of $10$ different classes of natural objects such as horse, airplane, automobile, etc. Furthermore, we perform evaluations on the high-resolution ImageNette dataset~\cite{imagenette} which is a $10$-class subset of the original large-scale ImageNet dataset~\cite{deng2009imagenet}. Finally, we evaluate the $K$-NN model on a tabular dataset, namely the UCI Spambase dataset~\cite{Dua:2019},
which consists of bag-of-words feature vectors on E-mails and determines whether the message is spam or not. The dataset contains 4,601 data cases, each of which is a 57-dimensional input.
We use $0.1\%$ of the MNIST and CIFAR-10 training data to finetune our models; on ImageNette and Spambase, we use 1\% for finetuning.
For evaluations on DNNs, we choose the CNN architecture from~\cite{gu2019badnets} on MNIST and the ResNet used in~\cite{cohen2019certified} on CIFAR-10, whereas for ImageNette, we use the standard ResNet-18~\cite{he2016deep} architecture.

\subsubsection{Training Protocol}
We set the number of sampled noise vectors (i.e. augmented datasets) to $N=1,000$ on MNIST and CIFAR, and $N=200$ on ImageNette, leading to an ensemble of $1,000$ and $200$ models, respectively.
The added smoothing noise is sampled from the Gaussian distribution with location parameter $\mu=0$ and scale $\sigma=0.5$ for MNIST and Spambase. For CIFAR-10 and ImageNette we use $\mu=0$ and set the scale to $\sigma=0.2$.
The impact of different $\sigma$ is shown in Section~\ref{sec:exp-acc-radius}.
The confidence intervals for the binomial distribution are calculated with an error rate of $\alpha=0.001$.
For the KNN models, we use $K=3$ neighbors and set the number of similarity levels to $L=200$, meaning that the similarity scores according to euclidean distance are quantized into 200 distinct levels.

\subsubsection{Baselines of Empirical Backdoor Removal Based Defenses}
\label{sec:baseline}

Since this is the first paper providing rigorous certified robustness against backdoor attacks, there is no baseline that allows a comparison of the certified accuracy. We remark that a technical report~\cite{wang2020certifying} directly applies the randomized smoothing technique to certify robustness against backdoors without evaluation or analysis.
However, as we will show in Section~\ref{sec:test-time augmentation}, directly applying randomized smoothing
without deterministic test-time augmentation does not provide high certified robustness.
We will, on the other hand, compare our empirical robust accuracy with the state-of-the-art empirical defenses. We briefly review these defenses in the following.

{\bf Activation clustering (AC)}~\cite{chen2018detecting} extracts the activation of the last hidden layer of a trained model and uses clustering analysis to remove training instances with anomalies. We use the default parameter setting provided in the Adversarial Robustness Toolbox (ART)~\cite{nicolae2018adversarial}.
{\bf Spectral Signature (Spectral)}~\cite{tran2018spectral} uses matrix decomposition on the feature representations to detect and remove training instances with anomalies. We again use the default parameter setting provided in ART.
{\bf Sphere}~\cite{steinhardt2017certified} performs dimensionality reduction and removes instances with anomalies in the lower dimensions. The top-15\% anomaly instances are removed.
{\bf Neural Cleanse (NC)}~\cite{wang2019neural} first reverse-engineers a ``pseudo-trigger'' for each class. Then, to detect and remove anomaly instances, the distances between each instance with and without the pseudo-trigger are compared, and the most similar ones are recognized as anomaly instances. We use pixel-level distance as the distance metric, 100 epochs for trigger generation, and an initial $\lambda=0.01$ for MNIST and $\lambda=0.0001$ for CIFAR and ImageNette.
{\bf Statistical Contamination Analyzer (SCAn)}~\cite{tang2021demon} first performs an EM algorithm to decompose two subgroups over a small clean dataset. Then, for each class in the train set, the parameters of a mixture model for all the data are estimated, before we calculate the likelihood for anomaly detection. To identify the backdoored instances, we recognize the smaller set in the most anomalous class as the backdoored instances.
{\bf Mixup}~\cite{borgnia2021strong}, following the data augmentation technique in the paper, we use a 4-way mixup training algorithm to train the model over the train set. The convex coefficients are drawn from a Dirichlet distribution with $\alpha=1.0$.

The initial goal of all these approaches, with the exception of Mixup, is to \textbf{detect} backdoored instances, i.e., to determine whether there exists a trigger. To apply them as a defense (i.e., to train a clean model despite the existence of backdoored data), we make adaptations either following the original paper (AC, Spectral, Sphere and NC) or by our design (SCAn) so that we remove training data with anomalies detected by these approaches and retrain a clean model. Some detection cannot be adapted to the defense task, such as~\cite{xu2021detecting}, and are not included in the comparison.

\subsubsection{Evaluation Metrics}
We evaluate the model accuracy trained on the backdoored dataset with vanilla training and RAB training strategies.
In particular, we evaluate both the model performance on benign instances (benign accuracy) and backdoored instances for which the attack was successful against the vanilla model (empirical robust accuracy). With RAB, we are also able to calculate the \textbf{certified accuracy}, which means that the RAB model not only certifies that the prediction is the same as if it were trained on the clean dataset, but also that the prediction is equal to the ground truth. The certified accuracy is defined below.
\begin{equation}
    \mathrm{Certified\, Acc.} = \frac{1}{n}\abs{\{x_i\colon R_i > \|\delta\|_2 \land \hat{y}_i = y_i\}}
\end{equation}
where $R_i$ is the robus radius according to Eq.~\eqref{eq:gaussian_bound}, $\hat{y}_i$ is the predicted label, and $y_i$ is the ground truth for input $x_i$.

We emphasize that we only evaluate the backdoored test instances for which the attack is successful against the vanilla trained models, which is why the vanilla models always have 0\% empirical robust accuracy on these backdoored instances in Table~\ref{tab:result-dnn}. This is to evaluate against the effective backdoor attacks and better illustrate the comparison between RAB-trained models with vanilla and baseline backdoor defense models (empirical robust accuracy).
Such empirical robust accuracy of different methods serves as an upper bound for the certified accuracy.

\subsubsection{Backdoor Patterns}
We evaluate RAB against three representative backdoor attacks, namely a one-pixel pattern in the middle of the image, a four-pixel pattern, and blending a random, but fixed, noise pattern to the entire image~\cite{chen2017targeted}.
We visualize all backdoor patterns on different datasets in Fig.~\ref{fig:backdoor-examples}.
We control the perturbation magnitude of the attack via the $L_2$-norm of the backdoor patterns, setting $\|\delta\|_2=0.1$ for all attacks where $\delta$ is the backdoor pattern.
On MNIST, we inject 10\% backdoored instances and 5\% for CIFAR and ImageNette respectively.
If not described differently, the attack goal is to fool the model into predicting ``0'' on MNIST, ``airplane'' on CIFAR and ``tench'' on ImageNette.
In Appendix~\ref{sec:exp-a2a}, we also consider an all-to-all attack goal~\cite{gu2019badnets} so that the fooled model will change its prediction conditioned on the original label.

It is possible to use different backdoor patterns via optimization and other approaches.
However, since our goal is to provide \textit{certified} robustness against backdoor attacks, a task that is by definition agnostic to the specific backdoor pattern but only depends on the magnitude of the pattern and the number of backdoored training instances, we mainly focus on these representative backdoor patterns.
In addition, we only evaluate the backdoor attack to poison the dataset, while other attacks that interfere with the training process are not evaluated~\cite{ren2021simtrojan}, as RAB is a robust training pipeline against training data manipulation based poisoning attacks.

\subsection{Certified Robustness of DNNs against Backdoor Attacks}

In this section we evaluate RAB against backdoor attacks on different models and datasets. We present both the certified robust accuracy of RAB, as well we the empirical robust accuracy comparison between RAB and baseline defenses. Furthermore, we also present several ablation studies to further explore the properties of RAB.

\subsubsection{Certified Robustness with RAB}
We first evaluate the certified robustness of RAB on DNNs against different backdoor patterns on different datasets.
We also present the performance of RAB on benign instances and backdoored instances empirically. Table~\ref{tab:result-dnn} lists the benchmark results on MNIST, CIFAR-10, and ImageNette, respectively.
From the results, we can see that RAB achieves significantly non-trivial certified robust accuracy against backdoor attacks at a negligible cost of benign accuracy; while there are no certified results for any other method. The slight drop in benign accuracy results from training on noisy instances. However, this loss in benign accuracy is less than 3\% in most cases and is clearly outweighed by the achieved certified robust accuracy.
In particular, RAB achieves over 23\% \textit{certified accuracy} on the backdoored instances for MNIST and CIFAR-10, and around 12\% for ImageNette.
In other words, we can successfully certify for these instances that our model predicts the same result as if it were trained on the clean training set.
We run the experiment multiple times and show in Appendix~\ref{sec:exp-multi-run} that the standard deviation is less than 1\% in most cases. We also show the abstain rate of  certification in Appendix~
\ref{sec:exp-abs-rate} and observe that it is generally low.
If the abstain rate is high, we can perform the similar way as in Cohen et al.~\cite{cohen2019certified} to obtain a variation of our theorem to certify the radius by some margin.

\subsubsection{Empirical Robustness: without RAB vs. with RAB}
In addition to the certificates that RAB can provide,
RAB's training process also provides good
robustness accuracy \textit{empirically}, without
theoretical guarantees.

In Table~\ref{tab:result-dnn}, the
``RAB" column reports the empirical
robust accuracy --- \textit{how often can
a malicious input that successfully
attacks a vanilla model
trick RAB?} We can see that,
RAB achieves
high empirical
robust accuracy, and such empirical robust accuracy achieved by either RAB or other methods serves as an upper bound for the certified robust accuracy provided by RAB under the ``RAB-certified" column.
It is shown that RAB achieves around 40\% empirical robust accuracy on the backdoored instances for MNIST and CIFAR-10, and over 30\% for ImageNette.
In Appendix~\ref{sec:exp-adv-atk}, we also try an empirical adversarial attack on the RAB model and observe a similar behavior as on vanilla models.

\subsubsection{Comparison with State-of-the-art Empirical Backdoor Defenses}
\label{subsubsec:empirical-comparison}
Another line of research is to develop empirical methods to automatically detect and remove backdoored training instances.
{\em How does RAB compare with these methods?}
We empirically compare the robustness of RAB with other  state-of-the-art baseline methods introduced in Section~\ref{sec:baseline}, as shown in Table~\ref{tab:result-dnn}.
we observe that although RAB is not specifically designed for empirical defense, it achieves comparable empirical robust accuracy compared with these baseline methods.
RAB outperforms about half of the baselines methods on MNIST and ImageNette and all the baselines on CIFAR-10. Interestingly, our approach performs better on CIFAR-10 than on other tasks while other baselines usually perform badly on  CIFAR-10.
We attribute this observation to the fact that the benign accuracy on CIFAR-10 is comparably low, so that the baselines based on analyzing feature representations or on model reverse engineering are largely affected and the performance is thus worse.
By comparison, RAB only needs to add noise to smooth the training process without analyzing model properties, and is hence less affected by the model viability (similarly, the performance of Mixup is less affected too).

In addition, in Appendix~\ref{sec:exp-a2a}, we additionally evaluate the defenses against a more challenging \textit{all-to-all attack} where many baseline approaches fail, and RAB still achieves good performance.
We also show that our approach can be applied to an SVM model for three tabular datasets in Appendix~\ref{sec:exp-svm}, while existing approaches cannot work well since there is no distinct ``activation layer'' in a simple SVM model.
Furthermore, for very large attack perturbations, the certification will fail as  shown in Appendix~\ref{sec:exp-large-pert}; however, RAB still achieves non-trivial empirical robustness.

\begin{figure}
    \centering
    \begin{subfigure}[b]{.4\linewidth}
    \centering
    \includegraphics[width=\linewidth]{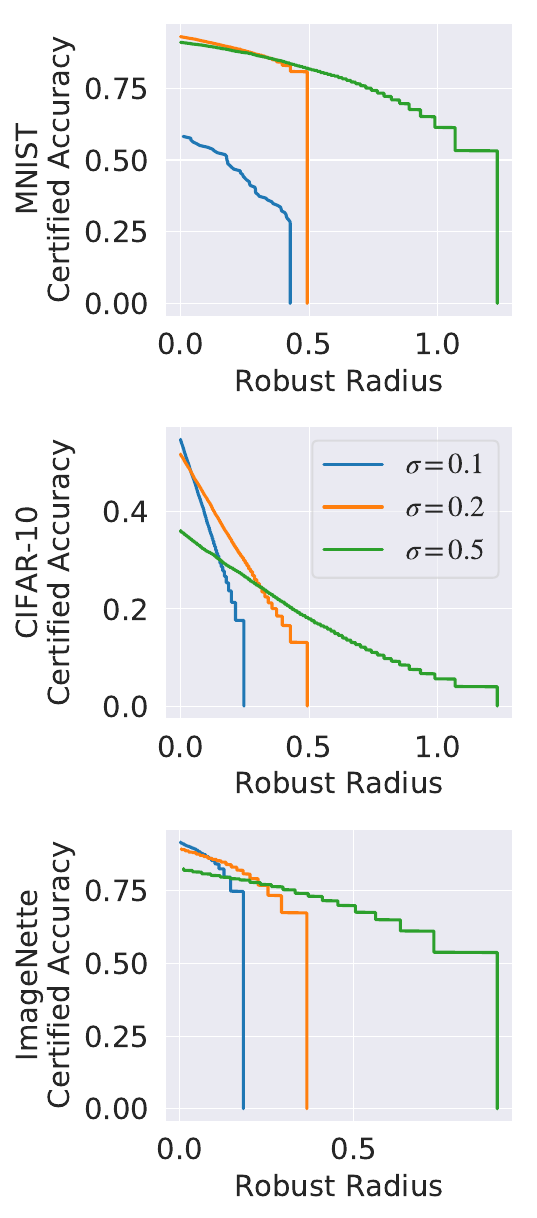}
    \caption{DNN}
    \label{fig:dnn-robust-accuracy}
    \end{subfigure}%
    \hfill
    \begin{subfigure}[b]{.4\linewidth}
    \centering
    \includegraphics[width=\linewidth]{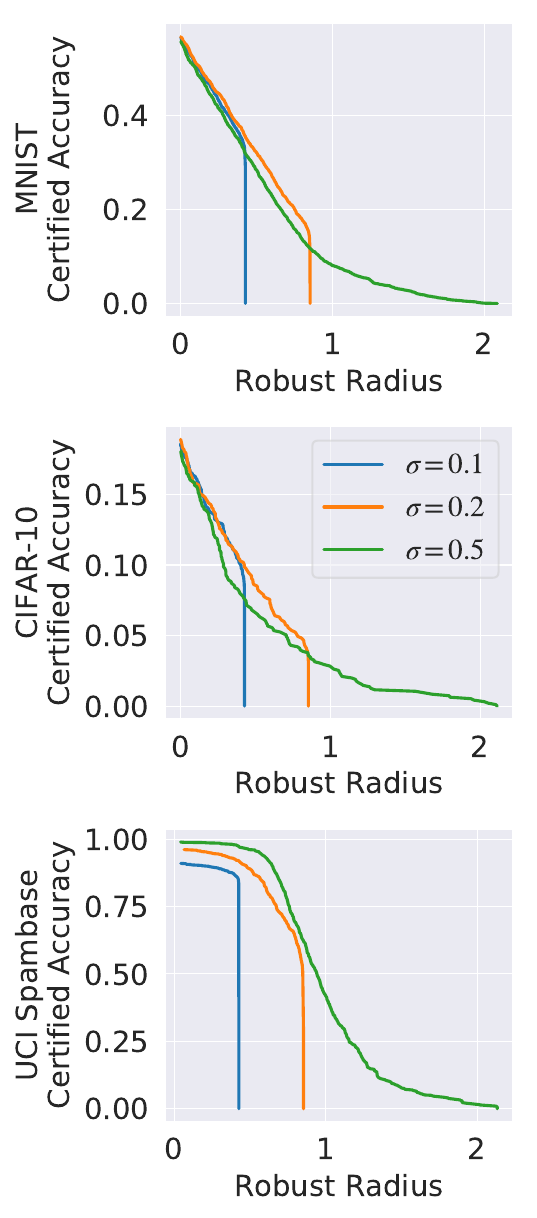}
    \caption{KNN}
    \label{fig:knn-robust-accuracy}
    \end{subfigure}%
    \caption{ Certified accuracy of DNN and KNN at different radii with different smoothing parameters $\sigma$ against blending attack. }
    \label{fig:result-radii}
\end{figure}

\subsubsection{Certified Accuracy Under different  Radii}
\label{sec:exp-acc-radius}
We further discuss how different certified radii affect the certified accuracy. In Fig.~\ref{fig:result-radii}, we present the certified accuracy as a function of the robust radius given different values for the smoothing parameter $\sigma$ against blending attack. The conclusions on other backdoor patterns are similar.

In the figures, we plot the certified accuracy of all test cases (instead of only on successfully attacked cases) so that the overall trend can be seen.
We can see that the certified accuracy decreases with increased radii and, at a certain point, it suddenly goes to zero, which aligns with existing observations on certified robustness against evasion attacks~\cite{cohen2019certified}. Furthermore, stronger noise harms the certified accuracy at a small radius, while improving it at a larger robust radius.
It is thus essential to choose an appropriate smoothing noise magnitude according to the task. The certified accuracy of KNN is comparatively low due to its simple structure, but it achieves non-trivial certified accuracy at a larger radius as we do not need Monte Carlo sampling which would result in a finite sampling error that decreases the certified robustness.

\begin{figure}
    \centering
    \includegraphics[width=0.7\linewidth]{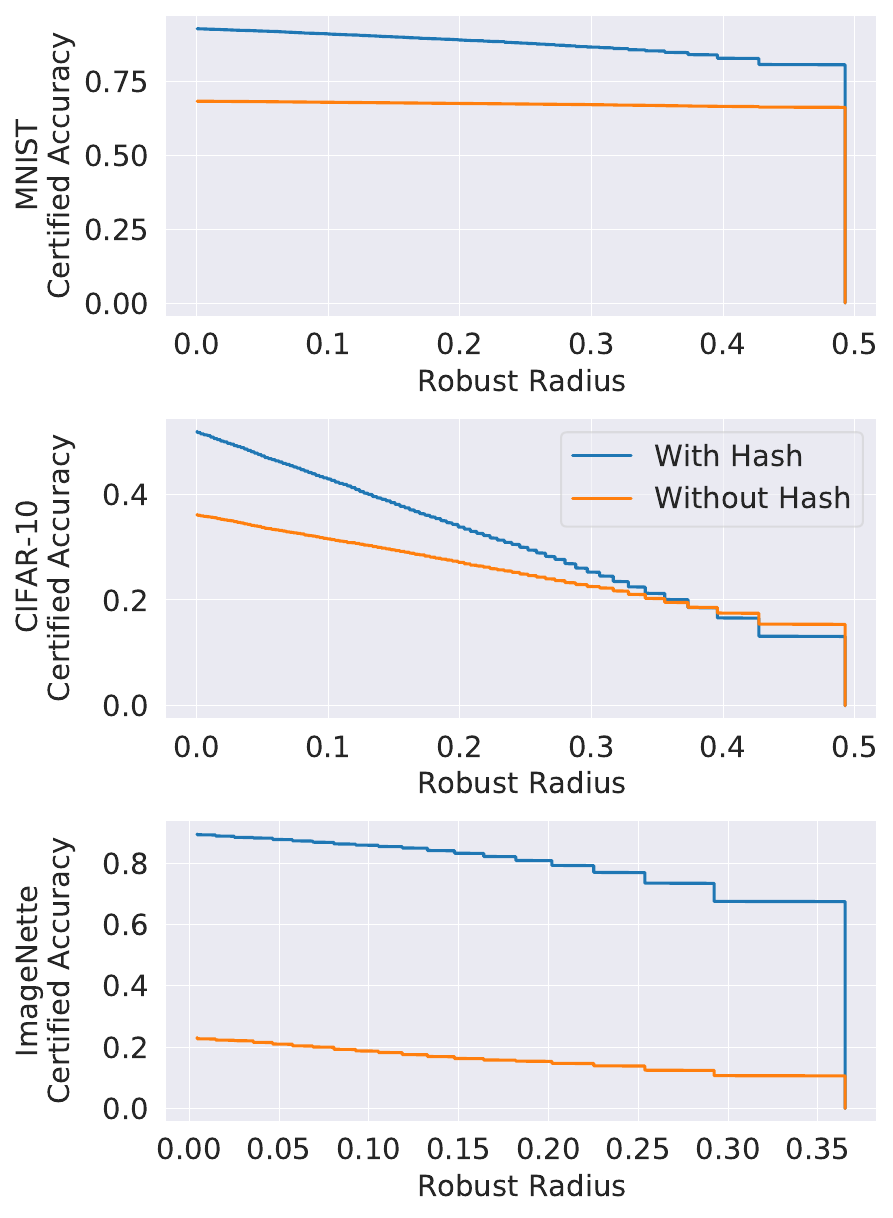}
    \caption{Comparison of the certified accuracy at different radii with and without the proposed deterministic test-time augmentation. The accuracy is evaluated against blending attack with smoothing parameter $\sigma=0.2$.}
    \label{fig:augmentation}
\end{figure}

\subsubsection{Ablation Study: Impact of Deterministic Test-time Augmentation}
\label{sec:test-time augmentation}
We compare the certification accuracy of RAB with and without deterministic test-time augmentation in Figure~\ref{fig:augmentation}.
We plot the certified accuracy of all test cases instead of only on successfully attacked cases to show the comparison on the entire dataset.
We observe that the certified accuracy significantly improves with the proposed hash function based deterministic test-time augmentation, especially at small certification radii and with a particularly large gap on ImageNette dataset
--- without the augmentation,  the certified accuracy is only around 20\%, while it increases to around 80\% with the augmentation.
This shows that it is important to include the test-time augmentation during inference, and directly adopting randomized smoothing may not provide satisfactory certified accuracy.
The detailed empirical and certified robust accuracies are shown in Appendix~\ref{sec:exp-no-aug}.


\begin{figure}
    \centering
    \includegraphics[width=0.5\linewidth]{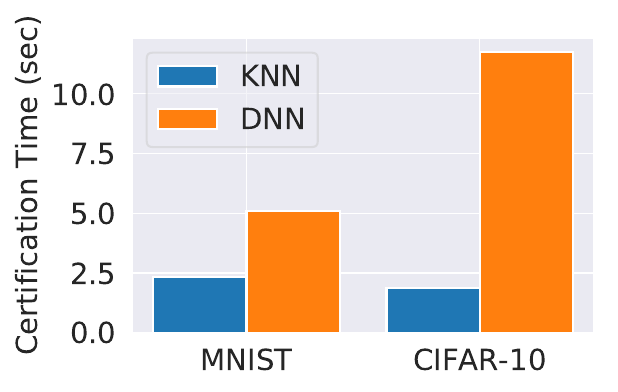}
    \caption{Runtime comparison for certifying one input.}
    \label{fig:speed}
\end{figure}

\subsection{Certified Robustness of KNN Models}

\label{subsec:certified-knn}

Here we present the benchmarks based on our proposed efficient algorithm for KNN models.
We perform experiments on the UCI spambase tabular datasets and show the results for K=3 in Table~\ref{tab:result-knn}.
The NC baseline relies on gradient-based reverse engineering, while Mixup relies on mixing label information during training, so these two methods
are not included here.
The other baselines use intermediate feature vectors in DNN models, which do not exist in KNN models.
Therefore, we use the output prediction vector as the feature vector.
From the results, we see that for KNN models, RAB achieves good performance for both empirical and certified robustness and outperforms all the baselines, indicating its advantages for specific domains.

This comparison might seem unfair at first glance, since the considered baselines are based on deep feature representations, which are absent in the KNN case.
However, firstly, we emphasize that none of the approaches, including RAB, use deep features for this comparison and have hence access to the same amount of information.
Secondly, this comparison reveals an important property of our approach: while the baselines struggle to handle ML models beyond DNN, RAB is applicable to a wider range of models and still yields non-trivial empirical and certified robust accuracy.
To enable a comparison for KNN models which is more favorable to the baselines, we consider kernel KNN with a CNN as the kernel function. From the table in Appendix~\ref{sec:exp-kernel-knn}, we see that for this scenario, some baselines indeed outperform RAB.

Figure~\ref{fig:speed} illustrates the runtime
of the exact algorithm for KNN vs. the sampling-based
method of DNN.
We observe that for certifying one input on KNN with $K=3$ neighbors, using the proposed \textit{exact} certification algorithm takes only 2.5 seconds, which is around 2-3 times faster than the vanilla RAB on MNIST and 6-7 times faster on CIFAR-10. In addition, the runtime is agnostic to the input size but related to the size of the training set. It would be interesting future work to design similar efficient certification algorithms for DNNs.
Nevertheless, the KNN algorithm remains slower than the algorithm without certification (which is 1000 times faster than the RAB DNN pipeline), and the improvement of running time is an important future direction.

\section{Related Work}
\label{sec:related_work}

In this section, we discuss current backdoor (poisoning) attacks on machine learning models and existing defenses.

\noindent\textbf{Backdoor attacks }
There have been several works developing optimal poisoning attacks against machine learning models such as SVM and logistic regression~\cite{biggio2012poisoning,li2016data}. Furthermore,~\cite{munoz2017towards} proposes a similar optimization-based poisoning attack against neural networks that can only be applied to shallow MLP models. In addition to these optimization-based poisoning attacks, the backdoor attacks are shown to be very effective against deep neural networks~\cite{chen2017targeted,gu2019badnets}. The backdoor patterns can be either static or generated dynamically~\cite{yang2017generative}. Static backdoor patterns can be as small as one pixel, or as large as an entire image~\cite{chen2017targeted}.

\noindent\textbf{Empirical defenses against backdoor attacks }
Given the potentially severe consequences caused by backdoor attacks, multiple defense approaches have been proposed.
NeuralCleanse~\cite{wang2019neural} proposes to detect the backdoored models based on the observation that there exists a ``short path'' to make an instance to be predicted as a malicious one.
\cite{chen2019deepinspect} improves upon the approach by using model inversion to obtain training data, and then applying GANs to generate the ``short path'' and apply anomaly detection algorithm as in Neural Cleanse. Activation Clustering~\cite{chen2018detecting} leverages the activation vectors from the backdoored model as features to detect backdoor instances.
Spectral Signature~\cite{tran2018spectral} identifies the ``spectral signature'' in the activation vector for backdoored instances.
STRIP~\cite{gao2019strip} proposes to identify the backdoor instances by checking whether the model will still provide a confident answer when it sees the backdoor pattern.
SentiNet~\cite{chou2018sentinet} leverages computer vision techniques to search for the parts in the image that contribute the most to the model output, which are very likely to be the backdoor pattern.
In~\cite{ma2019datapoisoning}, differential privacy has been leveraged as a defense against poisoning attacks. Note that RAB can not guarantee that the trained models are differentially private, although both aim to decrease the model sensitivity intuitively.
A further empirical defense against backdoor attacks is proposed in~\cite{hayase2021spectre} using covariance estimation with the aim of amplifying the spectral signature of backdoored instances.

\textbf{Certified Defenses against poisoning attacks }
Another interesting application of randomized smoothing is presented in~\cite{rosenfeld2020certified} to certify the robustness against label-flipping attacks and randomize the entire training procedure of the classifier by randomly flipping labels in the training set. This work is orthogonal to ours in that we investigate the robustness with respect to perturbations on the training inputs rather than labels.
In a further line of work on provable defenses against poisoning attacks,~\cite{levine2021deep} proposes an ensemble method, deep partition aggregation (DPA). Similar to our work, DPA is related to randomized smoothing, however, in contrast to our work, the goal is to certify the number of poisoned instances for which the prediction remains unaffected.
Similarly,~\cite{jia2021intrinsic} use an ensemble technique to certify robustness against poisoning attacks. This is also orthogonal to ours as it certifies the number of poisoned instances, rather than the trigger size.
The same certification goal is considered in~\cite{jia2022certified}, but is restricted to nearest neighbor algorithms and derives an intrinsic certificate by viewing them as ensemble methods.
In addition to these works aiming to certify the robustness of a single model, \cite{yang2020end} provides a new way to certify the robustness of an end-to-end sensing-reasoning pipeline.
Finally,~\cite{xie2021crfl} propose a technique to certify robustness against backdoor attacks within the federated learning framework by controlling the global model smoothness.
Furthermore, a technical report also proposes to directly apply the randomized smoothing technique to certify robustness against backdoor attacks without any evaluation or analysis~\cite{wang2020certifying}.
In addition, as we have shown, directly applying randomized smoothing will not provide high certified robustness bounds. Contrary to that, in this paper, we first provide a unified framework based on randomized smoothing, and then propose the RAB robust training process to provide certified robustness against backdoor attacks based on the framework. We provide the tightness analysis for the robustness bound, analyze different smoothing distributions, and propose the hash function-based model deterministic test-time augmentation approach to achieve good certified robustness. In addition, we analyze different machine learning models with corresponding properties such as model smoothness to provide guidance to further improve the certified robustness.

\section{Limitations}
\label{sec:limitations}
One major limitation of RAB is that it introduces non-negligible runtime overhead. To certify the robustness, we need to train and evaluate multiple models (here, 1000 for MNIST/CIFAR-10 and 200 for ImageNette), which is expensive despite the fact that it is parallelizable and can be speeded up with multiple GPUs.
Nevertheless, with our polynomial-time KNN algorithm, we have shown a first step towards mitigating the computational cost and leave further endeavors in this direction as future work.

Another limitation is the defender's knowledge of the attack. Indeed, to \emph{certify} the robustness, the defender needs to
know
1) an upper bound on the backdoor trigger magnitude (in terms of an $L_p$ norm), 2) an upper bound on the number of poisoned training instances, and, 3) control over the training process.
However, to use RAB only as a defense (i.e. without any certificate), the defender only needs to control the training process while 1) and 2) are not needed.
The assumption 3) restricts RAB to be a robust training algorithm given an untrusted dataset. In other words, RAB cannot be used to defend against backdoor attacks that interfere with the training process (e.g., \cite{ren2021simtrojan}).

\section{Discussion and Conclusion}
\label{sec:conclusions}
In this paper, we aim to propose a unified smoothing framework to certify the model robustness against different attacks. In particular, towards the popular backdoor poisoning attacks, we propose the first robust smoothing pipeline RAB as well as a \emph{model deterministic test-time augmentation} mechanism to certify the prediction robustness against diverse backdoor attacks. In particular, we evaluate our certified robustness against backdoors on DNNs and KNN models. In addition, we propose an \textit{exact} algorithm for KNN models without requiring to sample from the smoothing distributions. We provide comprehensive benchmarks of certified model robustness against backdoors on diverse datasets, which we believe will provide the \emph{first set} of certified robustness against backdoor attacks for future work to compare with, and hopefully our results and analysis will inspire a new line of research on tighter certified accuracy against backdoor attacks.

\section*{Acknowledgement}
This work is partially supported by NSF grant No.1910100, NSF CNS 2046726, C3 AI, and the Alfred P. Sloan Foundation.
CZ and the DS3Lab gratefully acknowledge the support from the Swiss State Secretariat for Education, Research and Innovation (SERI) under contract number MB22.00036 (for European Research Council (ERC) Starting Grant TRIDENT 101042665), the Swiss National Science Foundation (Project Number 200021\_184628, and 197485), Innosuisse/SNF BRIDGE Discovery (Project Number 40B2-0\_187132), European Union Horizon 2020 Research and Innovation Programme (DAPHNE, 957407), Botnar Research Centre for Child Health, Swiss Data Science Center, Alibaba, Cisco, eBay, Google Focused Research Awards, Kuaishou Inc., Oracle Labs, Zurich Insurance, and the Department of Computer Science at ETH Zurich.

\bibliographystyle{IEEEtranS}
\bibliography{IEEEabrv,main}
\begin{appendices}

\renewcommand*{\thelem}{\Alph{section}.\arabic{lem}}
\renewcommand*{\thecor}{\Alph{section}.\arabic{cor}}
\renewcommand*{\thedefn}{\Alph{section}.\arabic{defn}}
\renewcommand*{\therem}{\Alph{section}.\arabic{rem}}
\renewcommand{\thetable}{\Alph{section}.\arabic{table}}

\section{Proofs}
Here we provide the proofs for the results stated in the main part of the paper. We write $\alpha(\phi) = \alpha(\phi;\,\Prob_0)$ and $\beta(\phi) = \beta(\phi;\,\Prob_0,\,\Prob_1)$ for type-I and -II error probabilities.
\subsection{Proof of Theorem~\ref{thm:main}}
\label{appendix:main_theorem_proof}
{\it Preliminaries and Auxiliary Lemmas:}
Central to our theoretical results are likelihood ratio tests which are statistical hypothesis tests for testing whether a sample $x$ originates from a distribution $X_0$ or $X_1$. These tests are defined as
\begin{align}
\small
    \phi(x) =
        \begin{cases}
            1 &\mathrm{if}\, \Lambda(x) > t,\\
            q &\mathrm{if}\, \Lambda(x) = t,\\
            0 &\mathrm{if}\, \Lambda(x) < t.\\
        \end{cases}
    \hspace{0.6em}
    \text{with}
    \hspace{0.6em}
    \Lambda(x) = \frac{f_{X_1}(x)}{f_{X_0}(x)},\label{eq:likelihood_ratio_test}
\end{align}
where $q$ and $t$ are chosen such that $\phi$ has significance $\alpha_0$, i.e.
    $\alpha(\phi)
    =\Prob_0\left(\Lambda(X) > t\right) + q\cdot \Prob_0(\Lambda(X) = t) = \alpha_0.$
\begin{lem}
    \label{lem:sandwich}
    Let $X_0$ and $X_1$ be two random variables with densities $f_0$ and $f_1$ with respect to a measure $\mu$ and denote by $\Lambda$ the likelihood ratio $\Lambda(x) = f_1(x) / f_0(x)$. For $p\in[0,\,1]$ let $t_p:=\inf\{t\geq 0\colon \, \Prob(\Lambda(X_0) \leq t) \geq p\}$. Then it holds that
    \begin{equation}
        \label{eq:sandwich}
            \small
        \Prob\left(\Lambda(X_0) < t_p\right) \leq p \leq \Prob(\Lambda(X_0) \leq t_p).
    \end{equation}
\end{lem}
\begin{proof}
    We first show the RHS of inequality~(\ref{eq:sandwich}). This follows directly from the definition of $t_p$ if we show that the function $t\mapsto \Prob(\Lambda(X_0) \leq t)$ is right-continuous. Let $t\geq 0$ and let $\{t_n\}_n$ be a sequence in $\bR_{\geq 0}$ such that $t_n \downarrow t$. Define the sets $A_n:=\{x\colon\Lambda(x) \leq t_n\}$ and note that $\Prob(\Lambda(X_0) \leq t_n) = \Prob(X_0 \in A_n)$. Clearly, if $x\in \{x\colon\Lambda(x) \leq t\}$ then $\forall n\colon \Lambda(x) \leq t \leq t_n$ and thus $x\in\cap_n\,A_n$. If on the other hand $x\in\cap_n\,A_n$ then $\forall n\colon \Lambda(x) \leq t_n \to t$ as $n\to\infty$.
    Hence, we have that $\cap_n\, A_n = \{x\colon\Lambda(x) \leq t\}$ and thus $\lim_{n\to\infty}\Prob\left(\Lambda(X_0) \leq t_n\right) = \Prob\left(\Lambda(X_0) \leq t\right)$
    since $\lim_{n\to\infty}\Prob\left(X_0 \in A_n\right) = \Prob(X_0\in\cap_n A_n)$ for $A_{n+1} \subseteq A_{n}$.
    We conclude that $t\mapsto \Prob\left(\Lambda(X_0) \leq t\right)$ is right-continuous and in particular $\Prob\left(\Lambda(X_0) \leq t_p\right)\geq p$.
    We now show the LHS of inequality~(\ref{eq:sandwich}). For that purpose, consider the set $B_n:=\{x\colon\Lambda(x) < t_p - \nicefrac{1}{n}\}$ and let $B:=\{x\colon\Lambda(x) < t_p\}$. Clearly, if $x\in \cup_n B_n$, then $\exists n$ such that $\Lambda(x) < t_p - \nicefrac{1}{n} < t_p$ and hence $x\in B$. If on the other hand $x\in B$, then we can choose $n$ large enough such that $\Lambda(x) < t_p - \nicefrac{1}{n}$ and thus $x\in\cup_n B_n$. It follows that $B = \cup_n B_n$. Furthermore, by the definition of $t_p$ and since for any $n\in\bN$ we have that $\Prob\left(X_0 \in B_n\right) = \Prob\left(\Lambda(X_0) < t_p - \nicefrac{1}{n}\right) < p$ it follows that $\Prob\left(\Lambda(X_0) < t_p\right) = \lim_{n\to\infty}\Prob\left(X_0 \in B_n\right) \leq p$
    since $B_n \subseteq B_{n+1}$. This concludes the proof.
\end{proof}
\begin{lem}
    \label{lem:np_aux}
    Let $X_0$ and $X_1$ be random variables taking values in $\cZ$ and with probability density functions $f_0$ and $f_1$ with respect to a measure $\mu$. Let $\phi^*$ be a likelihood ratio test for testing the null $X_0$ against the alternative $X_1$. Then for any deterministic function $\phi\colon\cZ\to [0,1]$ the following implications hold:
    \begin{enumerate}
        \item[i)] $\alpha(\phi) \geq 1 - \alpha(\phi^*) \Rightarrow 1 - \beta(\phi) \geq \beta(\phi^*)$
        \item[ii)] $\alpha(\phi) \leq \alpha(\phi^*) \Rightarrow \beta(\phi) \geq \beta(\phi^*)$
    \end{enumerate}
\end{lem}
\begin{proof}
    We first show $(i)$. Let $\phi^*$ be a likelihood ratio test as defined in~(\ref{eq:likelihood_ratio_test}).
    Then, for any other test $\phi$ we have
    \begin{equation}
    \small
    \begin{aligned}
        \small
        &1-\beta(\phi^*) - \beta(\phi)=\\
        &=\int_{\Lambda > t} \phi f_1 d\mu + \int_{\Lambda \leq t} \left(\phi - 1\right)f_1 d\mu + q \int_{\Lambda = t} f_1 d\mu\\
        &=\int\displaylimits_{\Lambda > t} \phi \Lambda f_0 d\mu + \int\displaylimits_{\Lambda \leq t} \underbrace{\left(\phi - 1\right)}_{\leq 0}\Lambda f_0 d\mu + q \int\displaylimits_{\Lambda = t} \Lambda f_0 d\mu\\
        &\geq t\cdot \left[\int\displaylimits_{\Lambda > t} \phi f_0 d\mu + \int\displaylimits_{\Lambda \leq t} \left(\phi - 1\right) f_0 d\mu + q \int\displaylimits_{\Lambda = t} f_0 d\mu\right]\\
        &= t\cdot\left[\alpha(\phi) - (1 - \alpha(\phi^*))\right]\geq 0
    \end{aligned}
    \end{equation}
    with the last inequality following from the assumption and $t\geq 0$. Thus, $(i)$ follows; $(ii)$ can be proved analogously.
\end{proof}

\begin{proof}[Proof of Theorem~\ref{thm:main}]
    We first show the existence of a likelihood ratio test $\phi_A$ with significance level $1 - p_A$. Let $Z':= (\Omega_x,\Delta) + Z$ and recall that the likelihood ratio $\Lambda$ between the densities of $Z$ and $Z'$ is given by $\Lambda(z) = \frac{f_{Z'}(z)}{f_Z(z)}$ and let $X' := \Omega_x + X$ and $D' = \Delta + D$.
    Furthermore, for any $p\in[0,\,1]$, let $t_p := \inf\{t\geq0\colon \Prob(\Lambda(Z) \leq t) \geq p \}$ and
    \begin{equation}
        \small
        q_p = \begin{cases}
            0 &\text{ if } \Prob\left(\Lambda(Z) = t_p\right) = 0,\\
            \frac{\Prob(\Lambda(Z) \leq t_p) - p}{\Prob(\Lambda(Z) = t_p)} & \text{ otherwise}.
        \end{cases}
    \end{equation}
    Note that by Lemma~\ref{lem:sandwich} we have that $\Prob(\Lambda(Z) \leq t_p) \geq p$ and
    \begin{equation}
        \small
        \begin{aligned}
            \Prob(\Lambda(Z) \leq t_p) &= \Prob(\Lambda(Z) < t_p) + \Prob(\Lambda(Z) = t_p)\\
            &\leq p + \Prob(\Lambda(Z) = t_p)
        \end{aligned}
    \end{equation}
    and hence $q_p \in [0,\,1]$. For $p\in[0,\,1]$, let $\phi_p$ be the likelihood ratio test defined in~\eqref{eq:likelihood_ratio_test} with $q\equiv q_p$ and $t\equiv t_p$.
    Note that $\phi_p$ has type-I error probability $\alpha(\phi_p) = 1-p$.
    Thus, the test $\phi_A \equiv \phi_{p_A}$ satisfies $\alpha(\phi_A) = 1-p_A$. It follows from assumption~(\ref{eq:confidence}) that $\Prob_{X,\,D}(h(x+X,\,\cD + D) = y_A) = q(y_A\lvert\,x,\,\cD) \geq 1 - \alpha(\phi_A)$ and thus, by applying the first part of Lemma~\ref{lem:np_aux} to the functions $\phi(z) \equiv \Id_{\{h((x,\,\cD) + z)=y_A\}}(z)$ and $\phi^*\equiv\phi_A$, it follows that
    \begin{align}
        \label{eq:main_proof_lower_bound}
        \begin{split}
            q(y_A\lvert\,x + \Omega_x,\,\cD + \Delta)
            &=1-\beta(\phi) \geq \beta(\phi_A).
        \end{split}
    \end{align}
    Similarly, the likelihood ratio test $\phi_B \equiv \phi_{1-p_B}$ satisfies $\alpha(\phi_B) = p_B$ and, for $y\neq y_A$, it follows from the assumption~(\ref{eq:confidence}) that $\Prob_{X,\,D}(h(x+X,\,\cD + D) = y) = q(y\lvert\,x,\,\cD) \leq p_B = \alpha(\phi_B).$
    Thus, applying the second part of Lemma~\ref{lem:np_aux} to the functions $\phi(z) = \Id_{\{h((x,\,\cD) + z)=y\}}(z)$ and $\phi^*\equiv\phi_B$ yields
    \begin{align}
        \label{eq:main_proof_upper_bound}
        \begin{split}
            q(y\lvert\,x + \Omega_x,\,\cD + \Delta)
            & = 1 - \beta(\phi) \leq 1 - \beta(\phi_B).
        \end{split}
    \end{align}
    Combining~(\ref{eq:main_proof_lower_bound}) and~(\ref{eq:main_proof_upper_bound}) we see that, if $\beta(\phi_A) + \beta(\phi_B) > 1$,
    then it is guaranteed that $q(y_A\lvert\,x + \Omega_x,\,\cD + \Delta) > \max_{y\neq y_A}q(y\lvert\, x + \Omega_x,\,\cD + \Delta)$
    what completes the proof.
\end{proof}

\subsection{Proof of Theorem~\ref{thm:tightness}}
\begin{proof}
    We show tightness by constructing a base classifier $h^*$, such that the smoothed classifier is consistent with the class probabilities~(\ref{eq:confidence}) for a given (fixed) input $(x_0,\,\cD_0)$ but whose smoothed version is not robust for adversarial perturbations $(\Omega_x,\,\Delta)$ that violate~(\ref{eq:robustness_condition}).
    Let $\phi_A$ and $\phi_B$ be two likelihood ratio tests for testing the null $Z\sim\Prob_0$ against the alternative $Z + (\Omega_x,\,\Delta) \sim \Prob_1$ and let $\phi_A$ be such that $\alpha(\phi_A) = 1-p_A$ and $\phi_B$ such that $\alpha(\phi_B) = p_B$. Since $(\Omega_x,\,\Delta)$ violates~\eqref{eq:robustness_condition}, we have that $\beta(\phi_A) + \beta(\phi_B) \leq 1$.
    Let $p^*$ be given by
    \begin{align}
        p^*(y\lvert\,x,\,\cD) = \begin{cases}
            1 - \phi_A(x - x_0,\,\cD-\cD_0)& \,y = y_A\\
            \phi_B(x - x_0,\,\cD-\cD_0)& \,y = y_B\\
            \frac{1 - p^*(y_A\lvert\,x,\,\cD) - p(y_B\lvert\,x,\,\cD)}{C-2} &\,\mathrm{o.w.}
        \end{cases}
    \end{align}
    where the notation $\cD - \cD_0$ denotes subtraction on the features but not on the labels.
    Note that for binary classification, $C=2$ we have that $\phi_A = \phi_B$ and hence $p^*$ is well defined since in this case, by assumption $p_A + p_B = 1$.
    If $C > 2$, note that it follows immediately from the definition of $p^*$ that $\sum_k p^*(y\lvert\,x,\,\cD) = 1$. Note that, from the construction of $\phi_A$ and $\phi_B$ in the proof of Theorem~\ref{thm:main} (Appendix~\ref{appendix:main_theorem_proof}) that (pointwise) $\phi_A \geq \phi_B$ provided  $p_A + p_B \leq 1$.
    It follows that for $y\neq y_A,\,y_B$ we have $p^*(y\lvert\,x,\,\cD) \propto \phi_A - \phi_B \geq 0$. Thus, $p^*$ is a well defined (conditional) probability distribution over labels and $h^*(x,\,\cD):=\arg\max_y p^*(y\lvert\,x,\,\cD)$ is a base classifier. Furthermore, to see that the corresponding smoothed classifier $q^*$ is consistent with the class probabilities~(\ref{eq:confidence}), consider
    \begin{align}
        q^*(y_A\lvert\,x_0,\,\cD_0) &= \bE(1 - \phi_A(X,\,D)) =  p_A
    \end{align}
    and
    \begin{align}
        q^*(y_B\lvert\,x_0,\,\cD_0) &= \bE(\phi_B(X,\,D)) = \alpha(\phi_B) = p_B.
    \end{align}
    In addition, for any $y \neq y_A,\,y_B$, we have $q^*(y\lvert\,x_0,\,\cD_0) = (1-p_A-p_B)/(C-2) \leq p_B$ since by assumption $p_A + p_B \geq 1 -(C-2)\cdot p_B$. Thus, $q^*$ is consistent with the class probabilities~\eqref{eq:confidence}
    In addition, note that
    $q^*(y_A\lvert\,x_0 + \Omega_x,\,\cD_0 + \Delta) = 1 -\beta(\phi_A)$ and $\beta(\phi_B)=q^*(y_B\lvert\,x_0+\Omega_x,\,\cD_0+\Delta)$. Since by assumption $1 -\beta(\phi_A) < \beta(\phi_B)$ we see that indeed $y_A \neq g^*(x_0+\Omega_x,\,\cD_0 + \Delta)$.
\end{proof}

\subsection{Proof of Corollary~\ref{cor:gaussian}}
\begin{proof}
    We prove this statement by direct application of Theorem~\ref{thm:main}. Let $Z=(X,\,D)$ be the smoothing distribution for $q$ and let $\Tilde{Z}:= (\Omega_x,\,\Delta) + Z$ and $\Tilde{Z}':= (0,\,-\Delta) + \Tilde{Z}$. Correspondingly, let $\tilde{q}(y\lvert\,x,\,\cD) = q(y\lvert\,x + \Omega_x,\,\cD + \Delta)$. By assumption, we have that $\tilde{q}(y_A\lvert\,x,\,\cD) \geq p_A$ and $\max_{y\neq y_A}\Tilde{q}(y\lvert\,x,\,\cD) \leq p_B$.
    We will now apply Theorem~\ref{thm:main} to the smoothed classifier $\tilde{q}$. By Theorem~\ref{thm:main}, there exist likelihood ratio tests $\phi_A$ and $\phi_B$ for testing $\Tilde{Z}$ against $\Tilde{Z}'$ such that, if
    \begin{equation}
        \label{eq:robustness_condition_gaussian_apx}
        \beta(\phi_A) + \beta(\phi_B) > 1
    \end{equation}
    then it follows that $y_A = \arg\max_y\tilde{q}(y\lvert\,x,\,\cD - \Delta)$ The statement then follows, since $\tilde{q}(y\lvert\,x,\,\cD - \Delta)= \arg\max_y\tilde{q}(y\lvert\,x+\Omega_x,\,\cD+\Delta)$.
    We will now construct the corresponding likelihood ratio tests and show that~(\ref{eq:robustness_condition_gaussian_apx}) has the form~(\ref{eq:gaussian_bound}). Note that the likelihood ratio between $\Tilde{Z}$ and $\Tilde{Z}'$ at $z=(x,\,d)$ is given by $\Lambda(z) = \exp\left(\sum_{i=1}^n \langle d_i,\,-\delta_i \rangle_\Sigma + \frac12\langle \delta_i,\,\delta_i\rangle_\Sigma\right)$
    where $\Sigma = \sigma^2\Id_d$ and $\langle a,\,b\rangle_\Sigma := \sum_{i=1}^n a_i b_i / \sigma^2$.
    Thus, since singletons have probability $0$ under the Gaussian distribution, any likelihood ratio test for testing $\Tilde{Z}$ against $\Tilde{Z}'$ has the form
    \begin{align}
        \phi_t(z) =
            \begin{cases}
                1, &\,\Lambda(z) \geq t.\\
                0, &\,\Lambda(z) <  t.
            \end{cases}
    \end{align}
    For $p\in[0,\,1]$, let
    $t_p := \exp(\Phi^{-1}(p) \sqrt{\sum_{i=1}^n\langle\delta_i,\,\delta_i\rangle_\Sigma} - \frac{1}{2}\sum_{i=1}^n\langle\delta_i,\,\delta_i\rangle_\Sigma)$
    and note that $\alpha(\phi_{t_p}) = 1-p$ since
    $\alpha(\phi_{t_p}) = 1 - \Phi(\frac{\log(t_p) + \frac{1}{2}\sum_{i=1}^n\langle\delta_i,\,\delta_i\rangle_\Sigma}{\sqrt{\sum_{i=1}^n\langle\delta_i,\,\delta_i\rangle_\Sigma}})$
    where $\Phi$ is the CDF of the standard normal distribution.
    Thus, the test $\phi_A\equiv\phi_{t_A}$ with $t_A\equiv t_{p_A}$ satisfies $\alpha(\phi_A) = 1-p_A$ and the test $\phi_B\equiv\phi_{t_B}$ with $t_B\equiv t_{1-p_B}$ satisfies $\alpha(\phi_B) = p_B$.
    Computing the type II error probability of $\phi_A$ yields
    $\beta(\phi_A)=\Phi(\Phi^{-1}(p_A) - \sqrt{\sum_{i=1}^n\langle\delta_i,\,\delta_i\rangle_\Sigma})$.
    and, similarly, the type II error probability of $\phi_B$ is given by
    $\beta(\phi_B) = \Phi(\Phi^{-1}(1 - p_B) - \sqrt{\sum_{i=1}^n\langle\delta_i,\,\delta_i\rangle_\Sigma})$.
    Finally, we see that $\beta(\phi_A) + \beta(\phi_B) > 1$ is satisfied if and only if $\sqrt{\sum_{i=1}^n \|\delta_i\|_2^2} < \frac{\sigma}{2}(\Phi^{-1}(p_A) - \Phi^{-1}(p_B))$.
\end{proof}

\subsection{Proof of Corollary~\ref{cor:uniform}}
\begin{proof}
    We proceed analogously to the proof of Corollary~\ref{cor:gaussian} but with a uniform distribution on the feature vectors $D\sim\cU([a,\,b])$ and construct the likelihood ratio tests in the uniform case.
    Denote by $S:=\prod_{i=1}^n S_i,\,S_i:=\prod_{j=1}^d[a + \delta_{ij},\,b + \delta_{ij}]$ the support of $\Tilde{D} := \Delta + D$ and by $S':=\prod_{i=1}^n[a,\,b]^d$ the support of $\Tilde{D}' := D$.
    Note that the likelihood ratio between $\Tilde{Z}$ against $\Tilde{Z}'$ at $z=(x,\,w,\,v)$ for any $w \in S \cup S'$ is given by
    \begin{align}
        \Lambda(z) = \frac{f_{\Tilde{Z}'}(z)}{f_{\Tilde{Z}}(z)} = \frac{f_{\Tilde{W}'}(w)}{f_{\Tilde{W}}(w)} =
            \begin{cases}
                0 &\,w \in S\setminus S',\\
                1 &\,w \in S \cap S',\\
                \infty &\, w \in S'\setminus S.
            \end{cases}
    \end{align}
    and that any likelihood ratio test for testing $\Tilde{Z}$ against $\Tilde{Z}'$ has the form~\eqref{eq:likelihood_ratio_test}.
    We now construct such likelihood ratio tests $\phi_A,\,\phi_B$ with $\alpha(\phi_A) = 1 - p_A$ and $\alpha(\phi_B) = p_B$ by following the construction in the proof of Theorem~\ref{thm:main}. Specifically, we compute $q_A,\,t_A$ such that these type I error probabilities are satisfied. Notice that $p_0 := \Prob(\Tilde{W}\in S\setminus S')
        = 1 - \prod_{i=1}^n(\prod_{j=1}^d(1 - \frac{\abs{\delta_{ij}}}{b-a})_+)$
    where $(x)_+ = \max\{x,\,0\}$.
    For $t\geq 0$ we have $\Prob(\Lambda(\Tilde{Z}) \leq t) = p_0$ if $t < 1$ and $1$ otherwise.
    Thus $t_p:=\inf\{t\geq0\colon \Prob(\Lambda(Z) \leq t) \geq p \}$ is given by $t_p = 0$ if $p \leq p_0$ and $t_p = 1$ if $p>p_0$.
    We notice that, if $p_A \leq p_0$, then $t_A \equiv t_{p_A} = 0$. This implies that the type II error probability of the corresponding test $\phi_A$ is $0$ since in this case
    \begin{align}
        \beta(\phi_A) &= 1 - \Prob\left(\Lambda(\Tilde{Z}') > 0\right) - q_A\cdot \Prob\left(\Lambda(\Tilde{Z}') = 0\right)\\
        &\hspace{-2em}= 1 - \Prob\left(\Tilde{D}' \in S'\right) - q_A\cdot \Prob\left(\Tilde{D}' \in S\setminus S'\right) = 0.
    \end{align}
    Similarly, if $1- p_B \leq p_0$ then $t_B \equiv t_{p_B} = 0$ and we obtain that the corresponding test $\phi_B$ satisfies $\beta(\phi_B) = 0$.
    In both cases
    $\beta(\phi_A) + \beta(\phi_B) > 1$ can never be satisfied and we find that $p_A > p_0$ and $1-p_B > p_0$ is a necessary condition.
    In this case, we have that $t_A = t_B = 1$. Let $q_A$ and $q_B$ be defined as in the proof of Theorem~\ref{thm:main}
    \begin{align}
        q_A &:= \frac{\Prob(\Lambda(\Tilde{Z}) \leq 1) - p_A}{\Prob(\Lambda(\Tilde{Z}) = 1)} = \frac{1 - p_A}{1 - p_0},\\
        q_B &:= \frac{\Prob(\Lambda(\Tilde{Z}) \leq 1) - (1- p_B)}{\Prob(\Lambda(\Tilde{Z}) = 1)} = \frac{1 - (1 - p_B)}{1 - p_0}.
    \end{align}
    Clearly, the corresponding likelihood ratio tests $\phi_A$ and $\phi_B$ have significance $1-p_A$ and $p_B$ respectively. Furthermore, notice that
    \begin{align}
      \Prob\left(\Tilde{D}' \in S'\setminus S\right) &= \Prob\left(\Tilde{D} \in S\setminus S'\right) = p_0\\
      \Prob\left(\Tilde{D}' \in S'\cap S\right) &= \Prob\left(\Tilde{D} \in S'\cap S\right) = 1-p_0
    \end{align}
    and hence $\beta(\phi_A)$ is given by
    \begin{align}
        \beta(\phi_A) &= 1 - \Prob\left(\Lambda(\Tilde{Z}') > 1\right) - q_A\cdot \Prob\left(\Lambda(\Tilde{Z}') = 1\right)\\
        &\hspace{-2em}= 1 - p_0 - q_A \cdot (1 - p_0) = p_A - p_0.
    \end{align}
    and similarly
    \begin{align}
        \beta(\phi_B) &= 1- \Prob\left(\Lambda(\Tilde{Z}') > 1\right) - q_B\cdot \Prob\left(\Lambda(\Tilde{Z}') = 1\right)\\
        &\hspace{-2em}= 1 - p_0 - q_B \cdot (1 - p_0) = 1 - p_B - p_0.
    \end{align}
    Finally, the statement follows, since $\beta(\phi_A) + \beta(\phi_B) > 1$ if and only if $1 - \left(\frac{p_A - p_B}{2}\right) < \prod_{i=1}^n\left(\prod_{j=1}^d\left(1 - \frac{\abs{\delta_{ij}}}{b-a}\right)_+\right).$
\end{proof}

\section{Smoothed K-NN Classifiers}
\label{appendix:knn}
We first formalize $K$-NN classifiers which use quantizeed Euclidean distance as a notion of similarity. Specifically, let $B_1=[0,\,b_1),\,\ldots,B_L:=[b_{L-1},\infty)$ be similarity buckets with increasing $b_1 < b_2 < \ldots, b_{L-1}$ and associated similarity levels $\beta_1 > \beta_2 > \ldots > \beta_L$. Then for $x,\,x'\in\bR^d$ we define their similarity as $\kappa(x,x') := \sum_{l=1}^L\beta_l\Id_{B_l}(\left\|x-x'\right\|_2^2)$
where $\Id_{B_l}$ is the indicator function of $B_l$. Given a dataset $D=(x_i,\,y_i)_{i=1}^n$ and a test instance $x$, we define the relation
\begin{align}
    \label{eq:knn_binary_relation}
    x_i &\succeq x_j \iff
        \begin{cases}
            \kappa(x_i,\,x) > \kappa(x_j,\,x) &\text{ if } i > j\\
            \kappa(x_i,\,x) \geq \kappa(x_j,\,x) &\text{ if } i \leq j
        \end{cases}
\end{align}
which says that the instance $x_i$ is more similar to $x$, if either it has strictly larger similarity or, if it has the same similarity as $x_j$, then $x_i$ is more similar if $i < j$. With this binary relation, we define the set of $K$ nearest neighbours of $x$ as $I_K(x,\,\cD):= \{i\colon\, \left|\{j\neq i\colon\,x_j \succeq x_i\}\right| \leq K-1\} \subseteq [n]$
and summarize the per class votes in $I_K$ as a label tally vector $  \gamma_k(x,\,\cD) := \#\{i\in I_K(x,\,\cD)\colon\,y_i = k\}.$
The $K$-NN prediction is given by $\text{$K$-NN}(x,\,\cD) = \arg\max_k\gamma_k(x,\,\cD).$

\subsection{Proof of Theorem~\ref{thm:knn_complexity}}
\begin{proof}
Our goal is to show that we can compute the smoothed classifier $q$ with $Z=(0,\,D)$, $D\sim\prod_{i=1}^n\cN(0,\,\sigma^2\Id_d)$ and defined by the probability
\begin{equation}
    \label{eq:knn_proof_prob1}
    q(y\lvert\,x,\,\cD) = \Prob_D\left(\text{$K$-NN}(x,\,\cD + D) = y\right)
\end{equation}
in time $\cO(K^{2+C}\cdot n^2 \cdot L \cdot C)$. For ease of notation, let $X_i:=x_i + D^{(i)}$ and $S_i:=\kappa(X_i,\,x)$ and note that  $p_i^l:=\Prob\left(S_i=\beta_l\right) = F_{d,\lambda_i}\left(\frac{b_l}{\sigma^2}\right) - F_{d,\lambda_i}\left(\frac{b_{l-1}}{\sigma^2}\right)$
where $F_{d,\lambda_i}$ is the CDF of the non-central $\chi^2$-distribution with $d$ degrees of freedom and non-locality parameter $\lambda_i=\left\|x_i + \delta_i - x\right\|_2^2/\sigma^2.$
Note that we can write~(\ref{eq:knn_proof_prob1}) equivalently as $\Prob_D\left(\arg\max_{k'}\gamma_{k'}(x,\,\cD + D)) = y\right)    $
and thus $q(y\lvert\,x,\,\cD) = \sum_{\gamma\in\Gamma_{k}}\Prob_D\left(\gamma(x,\,\cD+D) = \gamma\right)$
with $\Gamma_k := \{\gamma\in[K]^C\colon\,\arg\max_{k'}\gamma_{k'}=k\}$. Next, we notice that the event that a tally vector $\gamma$ occurs, can be partitioned into the events that lead to the given $\gamma$ and for which instance $i$ has similarity $\beta_l$ and is in the top-$K$ but not in the top-$(K-1)$. We define these to be the boundary events $\cB_i^l(\gamma) := \{\forall c\colon\,\#\{j\in I_c\colon X_j \succeq X_i\} = \gamma_c,\,S_i = \beta_l\}$
where $I_c = \{i\colon y_i = c\}$. The probability that a given tally vector $\gamma$ occurs is thus given by $\Prob_D\left(\gamma(x,\,\cD+D) = \gamma\right) = \sum_{i=1}^n\sum_{l=1}^L\Prob\left(\cB_i^l(\gamma)\right).$

For fixed $i$ we notice that the for different classes, the events $\{\#\{j\in I_c\colon X_j \succeq X_i\} = \gamma_c\}$ are pairwise independent, conditioned on $\{S_i=\beta_l\}$. Writing $P_c^l(i,\,\gamma)$ for the conditional probability $\Prob\left(\#\{i\in I_c\colon y_i=c\}=\gamma_c\lvert \,S_i=\beta_l\right)$ yields $\Prob\left(\cB_i^l(\gamma)\right) = p_i^l\cdot\prod_{c=1}^C\,P_c^l(i,\,\gamma)$
and hence $q(y\lvert\,x,\,\cD) = \sum_{\gamma\in\Gamma_{k}}\sum_{i=1}^n\sum_{l=1}^L p_i^l\cdot\prod_{c=1}^C\,P_c^l(i,\,\gamma)$
which requires $\cO(K^C\cdot n\cdot L\cdot C)$ evaluations of $P_c^l$.
The next step is to compute the probabilities $P_c^l$. For that purpose, we need to open up the binary relation $\succeq$.
Suppose first that $y_i \neq c$.
Then the event that exactly $\gamma_c$ instances of class $c$ satisfy $X_j\succeq X_i$ is the same as the event that for some $r\leq \gamma_c$ exactly $r$ instances with index larger than $i$ have similarity strictly larger than $X_i$ and exactly $\gamma_c-r$ instances with an index smaller than $i$ have similarity larger or equal than $X_i$.
Now suppose that $y_i = c$.
Then, the event that exactly $\gamma_c$ instances of the same class $c$ satisfy $X_j\succeq X_i$ is the same as the event that for some $r\leq \gamma_c$ exactly $r$ instances with an index larger than $i$ have similarity strictly larger than $X_i$ and exactly $\gamma_c-r-1$ instances with an index smaller than $i$ have similarity larger or equal than $X_i$. This reasoning allows us to write $P_c^l$ in terms of functions
\begin{align}
    R_c^l(i,\,r)&:=\Prob\left(|\{j\in I_c\colon\,S_j > \beta_l,\,j>i\}|=r\right)\\
    Q_c^l(i,\,r)&:=\Prob\left(|\{j\in I_c\colon\,S_j \geq \beta_l,\,j<i\}|=r\right)
\end{align}
as
\begin{align*}
    P_c^l(i,\,\gamma)=\begin{cases}
        \sum_{r=0}^{\gamma_c} R_c^l(i, r)\cdot Q_c^l(i,\gamma_c-r) & y_i \neq c\\
        \sum_{r=0}^{\gamma_c-1} R_c^l(i, r)\cdot Q_c^l(i,\gamma_c-r-1) &y_i = c.
    \end{cases}
\end{align*}
The functions $R_l^c$ and $Q_l^c$ exhibit a recursive structure that we wish to exploit to get an efficient algorithm. Specifically, we write $\alpha_i^l:=\Prob(S_i \leq \beta_l) = \sum_{s=l}^L p_i^l$, and $\Bar{\alpha}_i^l:=1-\alpha_i^l$ and use the following recursion
\begin{align*}
    R_l^c(i-1,\,r) =
        \begin{cases}
            R_c^l(i,r) &y_{i}\neq c\\
            \Bar{\alpha}_{i}^l\cdot R_c^l(i,r-1) + \alpha_{i}^l\cdot R_c^l(i,r) &y_{i} = c
        \end{cases}
\end{align*}
starting at $i=n$ and $r=0$ and with initial values $R_c^l(i,\,-1) = 0$, $R_l^c(n,\,0) = 1$ and $R_l^c(n,\,r)= 0$ for $r\geq 1$. Similarly,
\begin{align}
    \label{eq:knn_q_func}
    Q_c^l(i+1,\,r) =
        \begin{cases}
            Q_c^l(i,\,r) & y_{i} \neq c\\
            \!\begin{aligned}
                &\Bar{\alpha}_{i}^{l+1}\cdot Q_c^l(i,\,r-1)\\
                &\hspace{2em} + \alpha_i^{l+1}\cdot Q_c^l(i,\,r)
            \end{aligned}
            & y_{i} = c
        \end{cases}
\end{align}
starting the recursion at $i=1$ and $r=0$ and with initial values $Q_c^l(i,\,-1) = 0$, $Q_c^l(1,\,0) = 1$ and $Q_c^l(1,\,r) = 0$ for $r \geq 1$. Evaluating $P_c^l$ requires $\cO(K)$ calls to $R_c^l$ and $Q_c^l$ each. The computation of $R_c^l$ and $Q_c^l$ can be achieved in $\cO(n\cdot K)$ if the values $\alpha_i^l$ are computed beforehand and stored separately (requiring $\cO(n\cdot L)$ computations). The entire computation has complexity $\cO(K^{C+2}\cdot n
^2 \cdot L \cdot C)$.
\end{proof}


\section{Additional Experimental Results}
\subsection{All-to-all Attacks}

\label{sec:exp-a2a}
\begin{table*}
	\caption{Evaluation on \textbf{DNNs} with different datasets with an all-to-all attack goal. We use $\sigma=0.5$ for MNIST and $\sigma=0.2$ for CIFAR-10 and ImageNette. ``Vanilla" denotes DNNs without RAB training and ``RAB-cert” presents certified accuracy of RAB. The highest empirical robust accuracies are {\bf bolded}.}
	\label{tab:result-dnn-all2all}
	\centering
	\resizebox{\linewidth}{!}{
	\begin{tabular}{l c c c c c c c c c c c c c}
		\toprule
        & \multirow{2}{*}[-2pt]{\makecell{Backdoor \\ Pattern}} & \multicolumn{2}{c}{Acc. on Benign Instances} & \multicolumn{8}{c}{Empirical Robust Acc.} & Certified Robust Acc.\\
        \cmidrule(lr){3-4}\cmidrule(lr){5-12}\cmidrule(lr){13-13}
		& & Vanilla & RAB & Vanilla & RAB & AC~\cite{chen2018detecting} & Spectral~\cite{tran2018spectral} & Sphere~\cite{steinhardt2017certified} & NC~\cite{wang2019neural} & SCAn~\cite{tang2021demon} & Mixup~\cite{borgnia2021strong} & RAB-cert \\
		\midrule
		\multirow{3}{*}{MNIST} & One-pixel & 91.5\% & 90.2\% & 0\% & \textbf{51.2\%} & 17.3\% & 3.0\% & 2.8\% & 28.4\% & 4.9\% & 37.1\% & 24.4\% \\
		& Four-pixel & 91.6\% & 91.3\% & 0\% & \textbf{60.3\%} & 16.1\% & 2.7\% & 1.8\% & 30.0\% & 1.8\% & 38.7\% & 39.9\% \\
		& Blending & 91.5\% & 91.2\% & 0\% & \textbf{59.7\%} & 15.4\% & 3.0\% & 1.8\% & 30.1\% & 4.7\% & 34.6\% & 39.1\% \\
		\midrule
		\multirow{3}{*}{CIFAR-10} & One-pixel & 58.4\% & 52.2\% & 0\% & 24.9\% & \textbf{26.7\%} & 5.7\% & 18.2\% & 13.2\% & 10.1\% & 19.7\% & 10.5\%\\
		& Four-pixel & 57.5\% & 52.1\% & 0\% & \textbf{25.1\%} & 11.2\% & 17.8\% & 18.3\% & 17.0\% & 13.3\% & 18.7\% & 11.6\% \\
		& Blending & 58.3\% & 52.1\% & 0\% & \textbf{24.8\%} & 10.0\% & 17.7\% & 15.9\% & 12.5\% & 10.7\% & 17.0\% & 10.9\% \\
		\midrule
		\multirow{3}{*}{ImageNette} & One-pixel & 92.5\% & 93.0\% & 0\% & 43.1\% & 32.8\% & 19.6\% & 41.2\% & 23.5\% & 23.5\% & \textbf{49.2\%} & 7.8\% \\
		& Four-pixel & 93.6\% & 93.0\% & 0\% & 37.5\% & 18.8\% & 18.8\% & 43.8\% & 26.3\% & 21.7\% & \textbf{58.3\%} & 18.7\%  \\
		& Blending & 95.0\% & 92.9\% & 0\% & 44.9\% & 46.9\% & 22.9\% & 34.7\% & 21.0\% & 14.3\% & \textbf{49.0\%} & 16.3\% \\
		\bottomrule
	\end{tabular}
	}
\end{table*}
\begin{table*}
	\caption{Evaluation on \textbf{DNNs} with different datasets with a large attack perturbation. We use $\sigma=0.5$ for MNIST and $\sigma=0.2$ for CIFAR-10 and ImageNette. ``Vanilla" denotes DNNs without RAB training and ``RAB-cert” presents certified accuracy of RAB. The highest empirical robust accuracies are {\bf bolded}.}
	\label{tab:result-dnn-largepert}
	\centering
	\resizebox{\linewidth}{!}{
	\begin{tabular}{l c c c c c c c c c c c c c}
		\toprule
        & \multirow{2}{*}[-2pt]{\makecell{Backdoor \\ Pattern}} & \multicolumn{2}{c}{Acc. on Benign Instances} & \multicolumn{8}{c}{Empirical Robust Acc.} & Certified Robust Acc.\\
        \cmidrule(lr){3-4}\cmidrule(lr){5-12}\cmidrule(lr){13-13}
		& & Vanilla & RAB & Vanilla & RAB & AC~\cite{chen2018detecting} & Spectral~\cite{tran2018spectral} & Sphere~\cite{steinhardt2017certified} & NC~\cite{wang2019neural} & SCAn~\cite{tang2021demon} & Mixup~\cite{borgnia2021strong} & RAB-cert \\
		\midrule
		\multirow{1}{*}{MNIST} & Large & 86.8\% & 86.5\% & 0\% & 42.3\% & 65.5\% & 8.1\% & 0.6\% & \textbf{70.9\%} & 11.9\% & 20.4\% & 0\% \\
		\midrule
		\multirow{1}{*}{CIFAR-10} & Large & 52.1\% & 44.8\% & 0\% & \textbf{27.2\%} & 20.88\% & 16.34\% & 11.96\% & 25.5\% & 8.6\% & 2.4\% & 0\% \\
		\midrule
		\multirow{1}{*}{ImageNette} & Large & 84.7\% & 81.6\% & 0\% & 46.4\% & 62.6\% & 36.3\% & 1.5\% & \textbf{74.9\%} & 55.5\% & 59.5\% & 0\%\\
		\bottomrule
	\end{tabular}
	}
\end{table*}
\begin{table*}
	\caption{Evaluation on \textbf{Kernel KNN} with different datasets. We use $\sigma=0.5$ for MNIST and $\sigma=0.2$ for CIFAR-10 and ImageNette. ``Vanilla" denotes DNNs without RAB training and ``RAB-cert” presents certified accuracy of RAB. The highest empirical robust accuracies are {\bf bolded}.}
	\label{tab:result-kernelknn}
	\centering
	\resizebox{\linewidth}{!}{
	\begin{tabular}{l c c c c c c c c c c c c c}
		\toprule
        & \multirow{2}{*}[-2pt]{\makecell{Backdoor \\ Pattern}} & \multicolumn{2}{c}{Acc. on Benign Instances} & \multicolumn{8}{c}{Empirical Robust Acc.} & Certified Robust Acc.\\
        \cmidrule(lr){3-4}\cmidrule(lr){5-12}\cmidrule(lr){13-13}
		& & Vanilla & RAB & Vanilla & RAB & AC~\cite{chen2018detecting} & Spectral~\cite{tran2018spectral} & Sphere~\cite{steinhardt2017certified} & NC~\cite{wang2019neural} & SCAn~\cite{tang2021demon} & Mixup~\cite{borgnia2021strong} & RAB-cert \\
		\midrule
		\multirow{3}{*}{MNIST} & One-pixel & 88.5\% & 78.2\% & 0\% & 35.7\% & 45.4\% & 53.0\% & 48.2\% & 53.0\% & 55.8\% & \textbf{59.5\%} & 18.0\% \\
		& Four-pixel & 88.5\% & 78.1\% & 0\% & 36.6\% & 50.6\% & 53.6\% & 48.3\% & \textbf{69.9\%} & 55.6\% & 52.2\% & 18.8\% \\
		& Blending & 88.4\% & 78.4\% & 0\% & 36.6\% &44.8\% & 52.4\% & 47.4\% & 51.5\% & \textbf{55.8\%} & 52.9\% & 18.8\%  \\
		\midrule
		\multirow{3}{*}{CIFAR-10} & One-pixel & 49.7\% & 46.5\% & 0\% & 21.6\% & 9.0\% & 24.9\% & 15.6\% & 16.5\% & 12.9\% & \textbf{25.1\%} & 11.3\% \\
		& Four-pixel & 49.5\% & 46.6\% & 0\% & 21.9\% & 15.9\% & 21.7\% & \textbf{22.7\%} & 13.4\% & 15.0\% & 19.2\% & 11.7\% \\
		& Blending & 49.8\% & 46.6\% & 0\% & 20.6\% & 17.0\% & 19.6\% & 15.1\% & 14.7\% & 16.8\% & \textbf{21.8\%} & 10.5\% \\
		\midrule
		\multirow{3}{*}{ImageNette} & One-pixel & 90.1\% & 88.6\% & 0\% & 35.3\% & \textbf{56.8\%} & 22.2\% & 28.4\% & 40.9\% & 19.3\% & 31.3\% & 8.8\% \\
		& Four-pixel & 90.7\% & 88.5\% & 0\% & 32.0\% & \textbf{52.2\%} & 29.6\% & 41.5\% & 34.0\% & 30.8\% & 27.7\% & 7.6\%\\
		& Blending & 91.5\% & 88.5\% & 0\% & 32.1\% & \textbf{33.3\%} & 17.2\% & 2.5\% & 23.0\% & 13.8\% & 21.8\% & 7.6\% \\
		\bottomrule
	\end{tabular}
	}
\end{table*}
In previous evaluations, the attack goal is to fool the model into a specific class. Here, we consider another attack goal that, on seeing the trigger pattern, the model will change its prediction from the $i$-th class to the $((i+1)\% C)$-th class, where $C$ is the number of classes. Different with the previous goal, the model here will need to recognize both the image and the trigger to make the malicious prediction. Thus, the defenses which assume that the backdoored model makes behavior only based on the backdoor trigger (e.g. NC) will intuitively not have a good performance.

The result of the all-to-all attack is shown in  Table~\ref{tab:result-dnn-all2all}. We observe that our approach achieves a similar performance for empirical and certified robustness. The performance on MNIST and ImageNette is slightly better compared with the standard attack, while on CIFAR-10 it decreases a little. As for the baselines, we can observe that the performance of Mixup is also consistent with that on the standard attack. This is understandable as Mixup also performs defense by processing the input and does not rely on model analysis. By comparison, the other baseline approaches based on model analysis does not achieve a good performance here. We owe it to the reason that in all-to-all attacks, the trained model needs to focus on both original image and the trigger pattern, so it is more difficult to detect the backdoors by model analysis than in standard attack where the model only focuses on the trigger pattern.

\subsection{Larger Perturbation}
\label{sec:exp-large-pert}
We consider a larger perturbation consisting of a $4\times 4$ trigger pattern with poison rate at 20\% and perturbation scale $||\delta_i||=4.0$ on MNIST and $||\delta_i||=4\sqrt{3}$ on CIFAR-10 and ImageNette (the $\sqrt{3}$ here comes from the fact that we add perturbation on all 3 channels). The results are shown in Table~\ref{tab:result-dnn-largepert}.
We can see that such strong perturbation is too large to be within our certification radius, which is a limit of our work. Therefore, the certified robust accuracy is 0. Nevertheless, we can still achieve some non-trivial empirical robustness and is comparable with baselines. This shows that our approach can be applied empirically to defend against strong backdoors with larger perturbation.

\subsection{Kernel-KNN}
\label{sec:exp-kernel-knn}
We evaluate the defense on KNNs with a kernel function. The kernel function is learned with the convolution neural network trained on the supervised task and uses the hidden representation of the last layer before output as the kernel output. Note that in this case, our exact KNN certification algorithm cannot be applied since the output with Gaussian variable cannot be analyzed with the kernel function. Therefore, we use the evaluation algorithm as in DNN to evaluate the certification performance.
As shown in Table~\ref{tab:result-kernelknn}, our approach achieves worse performance than on DNNs, which is understandable since KNN models are known to usually underperform DNN models. On the other hand, we observe that many baselines actually have a better performance than DNN. We view the reason to be that the baselines are based on the detection-and-removal algorithm. We find that the detection will only remove a subset of backdoored instances, so a trained DNN model will still be backdoored; however, any removal of backdoored training data will help the performance of KNN since fewer backdoored instances will be viewed as neighborhood, so the performance may improve. By comparison, RAB will not detect and remove instances and thus will not have a better performance on KNN.

\subsection{SVM-based model on tabular data}
\label{sec:exp-svm}
\begin{table*}
	\caption{Evaluation on \textbf{SVM} with different tabular datasets. We use $\sigma=0.5$ for Spam and $\sigma=0.2$ for Adult and Mushroom. ``Vanilla" denotes DNNs without RAB training and ``RAB-cert” presents certified accuracy of RAB. The highest empirical robust accuracies are {\bf bolded}.}
	\label{tab:result-svm}
	\centering
	\resizebox{\linewidth}{!}{
	\begin{tabular}{l c c c c c c c c c c c}
		\toprule
        & \multirow{2}{*}[-2pt]{\makecell{Backdoor \\ Pattern}} & \multicolumn{2}{c}{Acc. on Benign Instances} & \multicolumn{6}{c}{Empirical Robust Acc.} & Certified Robust Acc.\\
        \cmidrule(lr){3-4}\cmidrule(lr){5-10}\cmidrule(lr){11-11}
		& & Vanilla & RAB & Vanilla & RAB & AC~\cite{chen2018detecting} & Spectral~\cite{tran2018spectral} & Sphere~\cite{steinhardt2017certified} & SCAn~\cite{tang2021demon} & RAB-cert\\
		\midrule
		\multirow{3}{*}{Spam} & One-pixel & 91.8\% & 88.4\% & 0\% & \textbf{49.1\%} & 0\% & 18.3\% & 4.8\% & 12.9\% & 33.3\% \\
		& Four-pixel & 91.2\% & 88.6\% & 0\% & \textbf{48.2\%} & 0\% & 6.6\% & 7.4\% & 11.5\% & 32.1\% \\
		& Blending & 92.0\% & 89.2\% & 0\% & \textbf{44.7\%} & 0\% & 5.8\% & 5.8\% & 11.5\% & 29.8\% \\
		\midrule
		\multirow{3}{*}{Adult} & One-pixel & 79.0\% & 77.2\% & 0\% & \textbf{50.7\%} & 6.3\% & 15.3\% & 32.2\% & 8.4\% & 17.1\% \\
		& Four-pixel & 77.4\% & 73.1\% & 0\% & \textbf{53.0\%} & 5.4\% & 12.8\% & 14.4\% & 7.1\% & 21.5\% \\
		& Blending & 78.8\% & 76.4\% & 0\% & \textbf{55.9\%} & 8.0\% & 5.0\% & 11.6\% & 4.7\% & 26.1\% \\
		\midrule
		\multirow{3}{*}{Mushroom} & One-pixel & 87.5\% & 82.0\% & 0\% & \textbf{42.5\%} & 16.9\% & 0\% & 6.4\% & 17.3\% & 23.5\% \\
		& Four-pixel & 86.6\% & 80.1\% & 0\% & \textbf{42.2\%} & 14.2\% & 0\% & 2.8\% & 13.9\% & 22.5\% \\
		& Blending & 87.4\% & 81.4\% & 0\% & \textbf{43.5\%} & 13.1\% & 0\% & 11.1\% & 14.2\% & 24.0\% \\
		\bottomrule
	\end{tabular}
	}
\end{table*}
\begin{table}[htbp]
    \centering
    \caption{Robustness of RAB on \textbf{DNNs} with and without test-time augmentation.}
    \label{tab:no-aug-result}
    \resizebox{\linewidth}{!}{
    \begin{tabular}{c c c c c c}
		\toprule
         & \multirow{2}{*}[-2pt]{\makecell{Backdoor \\ Pattern}} & \multicolumn{2}{c}{
		With Aug} & \multicolumn{2}{c}{
		Without Aug}\\
		\cmidrule(lr){3-4}\cmidrule(lr){5-6}
         & & RAB & RAB-cert & RAB & RAB-cert \\
        \midrule
		\multirow{3}{*}{MNIST} & One-pixel & 41.2\% & 23.5\% & 27.0\% & 12.7\% \\
		 & Four-pixel & 40.7\% & 24.1\% & 27.4\% & 12.8\% \\
		 & Blending & 39.6\% & 23.1\% & 26.2\% & 12.1\%  \\
        \midrule
		\multirow{3}{*}{CIFAR-10} & One-pixel & 42.9\% & 24.5\% & 26.9\% & 15.2\%  \\
		 & Four-pixel & 44.4\% & 25.7\% & 28.4\% & 16.4\%  \\
		 & Blending & 42.8\% & 24.1\% & 27.8\% & 15.8\%  \\
        \midrule
		\multirow{3}{*}{ImageNette} & One-pixel & 38.6\% & 15.9\% & 22.7\% & 5.1\%  \\
		 & Four-pixel & 38.4\% & 12.6\% & 22.6\% & 8.2\%  \\
		 & Blending & 29.9\% & 9.2\% & 18.7\% & 4.1\% \\
		\bottomrule
    \end{tabular}}
\end{table}

As our certification for DNN can be applied to any machine learning model, we now evaluate RAB on three tabular data - UCI Spambase dataset (Spambase)~\cite{Dua:2019}, and ``Adult'' and ``Agaricus\_lepiota'' (Mushroom) in the Penn Machine Learning Benchmarks (PMLB) datasets\cite{romano2021pmlb}. These datasets are all binary classification tasks. Spambase contains 4,601 data points, with 57-dimensional input; Adult contains 48,842 data points with 14-dimensional input; Mushroom contains 8,145 data points with 22-dimensional input. We train a support vector machine (SVM) with RBF kernel using the default setting in scikit-learn toolkit~\cite{scikit-learn}. As for the baselines where activation vectors are required, we use the output prediction vector as its representation, since there are no hidden activation layers in an SVM model.

The result of the SVM dataset is shown in Table~\ref{tab:result-svm}. NC is not evaluated because it relies on anomaly detection among different classes, and therefore cannot be applied on these binary classification tasks; Mixup is not evaluated because it cannot be applied in the SVM training algorithm. We can see that our approach still achieves good robustness both empirically and certifiably. Meanwhile, the baseline approaches cannot perform well as they are designed specifically for deep neural networks. In the SVM case where they use the output as the representation vector, the detection performance cannot be good.

\subsection{With \& Without Test-time Augmentation}
\label{sec:exp-no-aug}
Table~\ref{tab:no-aug-result} shows the comparison of empirical and certified robustness with and without test-time augmentation. We see that the test-time augmentation indeed helps with the model robustness both empirically and certifiably.

\subsection{Abstain Rate}

\label{sec:exp-abs-rate}
\begin{table}[htbp]
    \centering
    \caption{The abstain rate of the certification on \textbf{DNNs}.}
    \label{tab:abstain-rate}
    \begin{tabular}{c c c}
		\toprule
         & Backdoor Pattern & Abstain Rate\\
        \midrule
		\multirow{3}{*}{MNIST} & One-pixel & 3.32\% \\
		 & Four-pixel & 3.21\% \\
		 & Blending & 3.02\% \\
        \midrule
		\multirow{3}{*}{CIFAR-10} & One-pixel & 5.59\% \\
		 & Four-pixel & 6.00\% \\
		 & Blending & 5.29\% \\
        \midrule
		\multirow{3}{*}{ImageNette} & One-pixel & 3.89\% \\
		 & Four-pixel & 4.08\% \\
		 & Blending & 1.90\% \\
		\bottomrule
    \end{tabular}
\end{table}
Table~\ref{tab:abstain-rate} shows the abstain rate of RAB against attacks. We see that in general, the abstain rate is relatively low and will not be a serious concern in the pipeline. Note that if the denial-of-service attack is indeed a concern, we can perform a similar way as in \cite{cohen2019certified} to prove a certified radius in which we can certify our defense rather than abstaining the input.

\subsection{Multiple Runs}

\label{sec:exp-multi-run}
\begin{table}[htbp]
    \centering
    \caption{The mean and standard deviation of the RAB robustness on DNNs with 5 runs.}
    \label{tab:multi-run}
    \begin{tabular}{c c c c}
		\toprule
         & Backdoor Pattern & RAB & RAB-cert\\
        \midrule
		\multirow{3}{*}{MNIST} & One-pixel & 40.79$\pm$0.72\% & 23.36$\pm$0.52\% \\
		 & Four-pixel & 40.27$\pm$0.87\% & 24.37$\pm$0.49\%\\
		 & Blending & 40.72$\pm$0.65\% & 23.58$\pm$0.88\% \\
        \midrule
		\multirow{3}{*}{CIFAR-10} & One-pixel & 42.66$\pm$0.29\% & 24.35$\pm$0.31\% \\
		 & Four-pixel & 42.56$\pm$0.32\% & 25.25$\pm$0.37\% \\
		 & Blending & 42.89$\pm$0.21\% & 23.95$\pm$0.17\%\\
        \midrule
		\multirow{3}{*}{ImageNette} & One-pixel & 38.64$\pm$0.80\% & 15.45$\pm$0.94\% \\
		 & Four-pixel & 37.23$\pm$0.69\% & 12.45$\pm$0.82\% \\
		 & Blending & 28.74$\pm$1.15\% & 9.20$\pm$1.40\%\\
		\bottomrule
    \end{tabular}
\end{table}
To see the stability of RAB, we run our algorithm 5 times and report the mean and standard deviation in Table~\ref{tab:multi-run}. We can see that the standard deviation is relatively small, indicating that our algorithm is stable.

\subsection{Adversarial Atacks on RAB Models}

\label{sec:exp-adv-atk}
\begin{figure}
    \centering
    \includegraphics[width=0.6\linewidth]{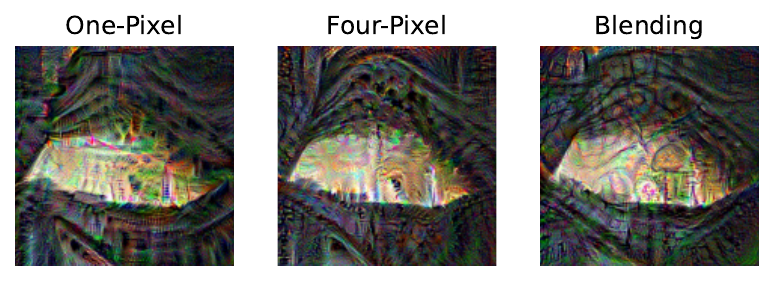}
    \caption{Adversarial examples against backdoored RAB model on the ImageNette dataset.}
    \label{fig:adv-atk}
\end{figure}
In \cite{sun2020poisoned}, the authors show that if they smooth a backdoored model, the robustified version will still be broken (i.e. with obvious adversarial pattern). We replicate the experiments on the RAB model by performing adversarial attacks against RAB model. In order to do attack, we use the PGD attack where the gradient is calculated by aggregating the gradient on all the trained models. In Figure~\ref{fig:adv-atk}, We show the results on ImageNette with $\varepsilon=60$ so that the pattern is the most clear. We observe that the adversarial examples look similar with those of unsmoothed model in \cite{sun2020poisoned}. Thus, the RAB pipeline is different with the smoothing process; rather, it is similar with an unsmoothed vanilla model.
\end{appendices}

\end{document}